%% file: paper.tex
\documentclass[]{hy}
\usepackage{wrapfig}
\usepackage{tabularx}
\usepackage{textcomp}
\usepackage{stfloats}
\usepackage{placeins}
\usepackage{url}
\usepackage{verbatim}
\usepackage{titlesec}
\usepackage{tocloft}
\usepackage{adjustbox}
\usepackage{multirow}
\usepackage{pifont}
\usepackage{tikz}
\usepackage{comment}
\usepackage{amsmath,amssymb,amsthm,mathtools} % define this before the line numbering.
\usepackage{colortbl}  %彩色表格需要加载的宏包
\usepackage{color}
\usepackage{booktabs} 
\usepackage{natbib} 
\setcitestyle{square,comma,numbers,sort&compress}
\usepackage{graphicx}    % 用于插入图片
\usepackage{subcaption} 
\usepackage{multirow} % 表格多行合并
\usepackage{booktabs} % 表格更好的线条
\usepackage{subcaption} % 支持 subtable 和 subfigure
\RequirePackage{xspace}
\makeatletter
\DeclareRobustCommand\onedot{\futurelet\@let@token\@onedot}
\def\@onedot{\ifx\@let@token.\else.\null\fi\xspace}

\makeatother

\definecolor{adptorange}{RGB}{248, 205, 172}
\definecolor{cmpblue}{RGB}{189, 215, 238}
\definecolor{cmpblue}{RGB}{189, 215, 238}

\definecolor{our_red}{RGB}{232,157,160}
\definecolor{our_blue}{RGB}{136,206,230}
\definecolor{our_orange}{RGB}{246,200,168}
\definecolor{our_green}{RGB}{178,211,164}

\definecolor{attn_code0}{RGB}{247,215,200}
\definecolor{attn_code1}{RGB}{238,169,139}
\definecolor{mlp_code0}{RGB}{204,201,221}
\definecolor{mlp_code1}{RGB}{102,95,153}

\definecolor{token_blue}{RGB}{84, 120, 140}
\input{math_commands}
\makeatletter
\renewcommand{\eqref}[1]{\textup{\tagform@{\ref{#1}}}}
\makeatother
\input{requirements/table}

\usepackage{pifont}

\usepackage{pifont}       % \ding{xx}
\usepackage{bbding}       % \Checkmark,\XSolid,... (需要和pifont宏包共同使用)
\usepackage{fontawesome}
\usepackage{xspace}

\usepackage{float}
\usepackage{enumitem}

\newlength\savewidth

\newcolumntype{x}[1]{>{\centering\arraybackslash}p{#1pt}}
\newcolumntype{y}[1]{>{\raggedright\arraybackslash}p{#1pt}}
\newcolumntype{z}[1]{>{\raggedleft\arraybackslash}p{#1pt}}

\renewcommand{\paragraph}[1]{\vspace{1mm}\noindent\textbf{#1}}
\usepackage{colortbl}
\usepackage{xcolor}
\usepackage{wrapfig}

\renewcommand{\paragraph}[1]{\vspace{1.25mm}\noindent\textbf{#1}}

\usepackage{algorithm}
\usepackage{algpseudocode}
\usepackage{listings}

\definecolor{codeblue}{rgb}{0.25, 0.5, 0.5}
\definecolor{codekw}{rgb}{0.35, 0.35, 0.75}
\lstdefinestyle{Pytorch}{
    language = Python,
    backgroundcolor = \color{white},
    basicstyle = \fontsize{9pt}{8pt}\selectfont\ttfamily\bfseries,
    columns = fullflexible,
    aboveskip=1pt,
    belowskip=1pt,
    breaklines = true,
    captionpos = b,
    commentstyle = \color{codeblue},
    keywordstyle = \color{codekw},
}

\definecolor{colSubject}{HTML}{D32F2F}   % Red: 核心主体 (Core Subjects)
\definecolor{colAction}{HTML}{F57C00}    % Orange: 核心动作 (Core Actions)
\definecolor{colDetail}{HTML}{388E3C}    % Green: 物体、修饰语 (Objects & Details)
\definecolor{colSpatial}{HTML}{1976D2}   % Blue: 空间与位置 (Spatial)
\definecolor{colMood}{HTML}{7B1FA2}      % Purple: 情绪与氛围 (Mood)
\definecolor{colKnow}{HTML}{AFB42B}      % Gold: 文本与认知 (Knowledge)

\usepackage{xcolor}
\usepackage{tikz}
\usetikzlibrary{tikzmark, calc, shadows.blur, shapes.geometric, fit, positioning}
\usepackage{xparse}
\pgfdeclarelayer{background}
\pgfdeclarelayer{floatbox}
\pgfdeclarelayer{foreground}
\pgfsetlayers{background,floatbox,main,foreground}

\newcounter{hcellcount}

\NewDocumentCommand{\hctext}{m}{\csname hctext@#1\endcsname}
\NewDocumentCommand{\sethctext}{mm}{\expandafter\gdef\csname hctext@#1\endcsname{#2}}

\definecolor{scoreRed}{RGB}{200, 0, 0}
\definecolor{grayText}{RGB}{120, 120, 120}

\definecolor{green}{HTML}{009000}
\definecolor{red}{HTML}{ea4335}

\newcommand{\D}{\mathcal{D}}
\newcommand{\Y}{\mathcal{Y}}
\renewcommand{\E}{\mathbb{E}}
\renewcommand{\KL}{\mathrm{KL}}
\newcommand{\piref}{\pi_{\mathrm{ref}}}
\newcommand{\pold}{\pi_{\theta^-}}
\newcommand{\pith}{\pi_\theta}

\newcommand{\piH}{\pi_{\mathrm H}}
\newcommand{\GH}{G_{\mathrm H}}
\newcommand{\Rtil}{\widetilde{p}}
\newcommand{\clip}{\operatorname{clip}}
\newcommand{\logit}{\operatorname{logit}}

\renewcommand{\sign}{\operatorname{sgn}}

\newtheorem{proposition}{Proposition}
\newtheorem{appendixproposition}{Proposition}[section]

\definecolor{cvblue}{rgb}{0.15, 0.45, 0.68}
\title{FlowBalance: Verifier-Grounded Self-Improvement from On-Policy Reasoning Experience}

\author[1,2]{Zixun Huang\textsuperscript{*}}
\author[1]{Kishan Panaganti\textsuperscript{*\(\dagger\)}}
\author[1]{Haitao Mi}
\author[1]{Leowei Liang}

\affiliation[1]{Tencent HY LLM Frontier}
\affiliation[2]{University of Pennsylvania}

\contribution{\textsuperscript{*}Equal contribution}
\contribution{\textsuperscript{\(\dagger\)}Project lead.
\\ zixunh@wharton.upenn.edu, kpb@global.tencent.com}

\projectlinks{\faGithub\ \textbf{GitHub:} \href{https://github.com/alexhuang13/FlowBalance}{github.com/alexhuang13/FlowBalance}\\[0.05cm]
\faRss\ \textbf{Blog:} \href{https://alexhuang13.github.io/FlowBalance-Blog/}{alexhuang13.github.io/FlowBalance-Blog/}}

\input{sections/abstract}

\date{\today}

\begin{document}
\thispagestyle{firstheader}
\maketitle
\pagestyle{empty}

\input{sections/introduction}
\input{sections/data_mixture}
\input{sections/method}
\input{sections/theory}
\input{sections/theory_sign_gated}
\input{sections/experiment}
\input{sections/related_work}
\input{sections/conclusion}

\bibliographystyle{unsrt}
\bibliography{paper}

% \printbibliography[heading=bibintoc]

\clearpage
\appendix
\input{sections/appendix}
\input{sections/appendix_sign_gated}
\input{sections/appendix_experiments}

\end{document}

%% file: math_commands.tex
\usepackage{amsmath,amsfonts,bm}

\def\eqref#1{equation~\ref{#1}}
\def\1{\bm{1}}

\DeclareMathAlphabet{\mathsfit}{\encodingdefault}{\sfdefault}{m}{sl}
\SetMathAlphabet{\mathsfit}{bold}{\encodingdefault}{\sfdefault}{bx}{n}

\newcommand{\E}{\mathbb{E}}

\newcommand{\KL}{D_{\mathrm{KL}}}
\newcommand{\Var}{\mathrm{Var}}

\newcommand{\Cov}{\mathrm{Cov}}
\DeclareMathOperator{\sign}{sign}

\usepackage{amsmath,amsfonts,bm}

\def\eqref#1{equation~\ref{#1}}
\def\1{\bm{1}}
\DeclareMathAlphabet{\mathsfit}{\encodingdefault}{\sfdefault}{m}{sl}
\SetMathAlphabet{\mathsfit}{bold}{\encodingdefault}{\sfdefault}{bx}{n}

%% file: requirements/table.tex
\usepackage{multirow}
\usepackage{diagbox}
\usepackage{makecell}
\usepackage{tabularx}
\usepackage{graphicx}

\usepackage{array}
\usepackage{rotating}

\definecolor{aliceblue}{rgb}{0.94, 0.97, 1.0}
\definecolor{citecolor}{HTML}{0071BC}
\definecolor{linkcolor}{HTML}{ED1C24}
\definecolor{darkgreen}{HTML}{539165}

\makeatletter
\newcommand{\thickhline}{%
 \noalign {\ifnum 0=`}\fi \hrule height 1pt
 \futurelet \reserved@a \@xhline
}
\makeatother

%% file: sections/abstract.tex
\abstract{
\noindent A reasoning model can improve from its own on-policy experience, but this inner loop is fragile: terminal verifiers provide reliable yet sparse supervision, while dense same-model guidance can reinforce false confidence or overconcentrate learning on a narrow solution mode. We introduce \textbf{FlowBalance}, a verifier-grounded self-improvement method that learns a normalized distribution over complete responses. For each on-policy trajectory, a frozen training-time view of the same policy uses privileged context to produce token-level log-probability gains, which are aggregated into a trajectory-level self-guidance score. FlowBalance calibrates this score with the verifier-derived group advantage: guidance is retained on positive-advantage trajectories, reversed on negative-advantage trajectories, and disabled when the rollout group provides no outcome preference. The resulting energy exponentially reweights a reference policy, and profiled trajectory balance fits the normalized target with one log-partition estimate per rollout group. This realizes \emph{outcome-calibrated self-guidance via trajectory balance}, without a separate token-level imitation loss. Our analysis establishes within-group contrast preservation, a minimum-change reverse-KL characterization, monotonic verifier control of target reward, and an exact correction against false-positive self-guidance on rejected responses. On mathematical reasoning, FlowBalance improves average performance over FlowRL on both Qwen3-4B and Qwen3-8B, while also improving training speed and stability, avoiding direct OPSD's response-length collapse, and exhibiting higher correct-strategy diversity in a controlled AIME24 diagnostic.
}

%% file: sections/introduction.tex
\section{Introduction}
\label{sec:introduction}

A central promise of post-training is that a reasoning model can improve from its own experience: sample several solutions, evaluate what worked, and update the policy so that the next round of experience is better. We use \emph{self-improvement} in this operational sense---repeated policy improvement from the model's own on-policy trajectories under training-time feedback---rather than to claim a fully autonomous or closed-loop system. For long-horizon reasoning, this inner loop must learn more than a higher mean reward. It must move probability toward correct reasoning, preserve useful alternative strategies, and avoid repeatedly sharpening one locally preferred trace~\citep{guan2026reasoningoverconfidence,kujawa2025neuralalgorithmic,ren2025learningdynamics}.

Two failure modes make this loop fragile. First, reinforcement learning with verifiable rewards (RLVR) provides dependable outcome grounding, but only through sparse terminal feedback~\citep{shao2024deepseekmath,yu2025dapo}. A response may contain hundreds or thousands of tokens while the verifier supplies one final reward; group-based objectives improve comparison across responses but still attach a response-level signal to every decision. Distributional methods such as FlowRL improve how this terminal evidence is translated into probability over complete responses~\citep{zhu2025flowrl}, yet outcome-only energies cannot exploit fine-grained evidence along a sampled reasoning path.

Second, the model can reassess its own sampled trajectories using a more informed training-time view. A frozen copy of the current policy can condition on a reference solution or task feedback and score every sampled token~\citep{zhao2026opsd,hubotter2026sdpo}. This same-model, privileged-hindsight signal is dense and inexpensive: it requires neither a larger external model nor additional generated responses. However, dense self-guidance is not automatically trustworthy. Because the scoring view observes information unavailable at inference, it may favor a plausible but ultimately rejected trajectory, shorten reasoning, suppress uncertainty-driven exploration, or concentrate updates around a narrow local mode~\citep{kim2026degradation,zhu2026manyfaces}. Blindly following such confidence creates \emph{self-confirmation}: the model's own erroneous preference becomes the next update's supervision.

We therefore ask a distributional self-improvement question:
\begin{quote}
\emph{Given on-policy reasoning experience, sparse verified outcomes, and dense but imperfect self-guidance, what normalized response distribution should the next policy learn?}
\end{quote}

We introduce \textbf{FlowBalance}, a verifier-grounded self-improvement operator that answers this question through \emph{outcome-calibrated self-guidance via trajectory balance}. The current policy generates a rollout group. A verifier supplies stopped group-relative advantages. A frozen privileged-hindsight view of the same policy then scores the already sampled tokens using training-only context. FlowBalance aggregates these scores into a trajectory-level guidance gain, uses the verifier to determine its direction, and fits the resulting reference-supported Gibbs target with profiled trajectory balance. The policy is optimized only through this complete-response distribution-matching objective; there is no separate token-level imitation loss.

\begin{figure*}[t]
    \vspace{-5em}
    \centering
    \includegraphics[width=0.99\textwidth]{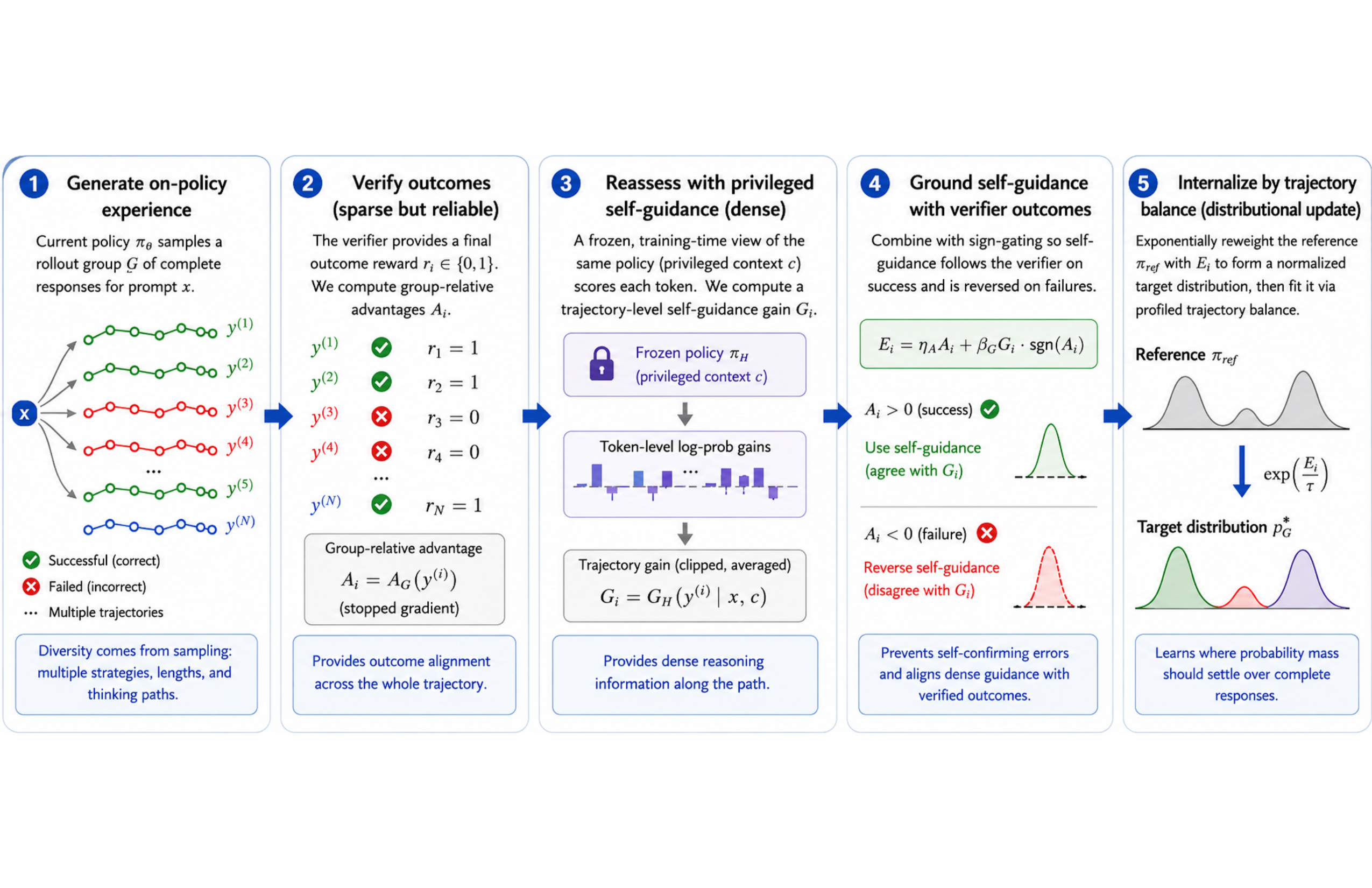}
    \vspace{-5em}
    \caption{\textbf{FlowBalance as a verifier-grounded self-improvement cycle.} The current policy generates on-policy reasoning experience (1); a verifier supplies sparse but reliable outcome feedback (2); a frozen privileged-hindsight view of the same policy supplies dense self-guidance on the sampled tokens (3); sign gating grounds that guidance in verified outcomes (4); and profiled trajectory balance internalizes the resulting normalized response distribution (5). Refreshing the frozen snapshot repeats the inner loop.}
    \label{fig:flowbalance-overview}
\end{figure*}

For a prompt $x$, let $\mathcal G=(y^{(1)},\ldots,y^{(N)})$ be the on-policy rollout group and let $A_i=A_{\mathcal G}(y^{(i)})$ be its stopped group-relative verifier advantage. The privileged-hindsight scoring view $\piH$ conditions on training-only context $c$ and yields clipped token-level log-probability gains relative to the reference policy. Their trajectory average is the self-guidance gain $\GH(y^{(i)}\mid x,c)$. FlowBalance defines
\begin{equation}
    E_{\mathrm{FlowBalance},\mathcal G}(y^{(i)}\mid x,c)
    =\eta_A A_i
     +\beta_G\GH(y^{(i)}\mid x,c)\sign(A_i).
    \label{eq:intro-energy}
\end{equation}
The verifier term anchors the direction of improvement. When $A_i>0$, positive self-guidance can reinforce a verified response; when $A_i<0$, the same positive guidance is reversed rather than allowed to self-confirm a failure; and when $A_i=0$, the dense branch is disabled. The guidance therefore refines the verifier's direction instead of overriding it.

Conditioned on the realized group, FlowBalance defines the reference-supported target
\begin{equation}
    p^\star_{\mathrm{FlowBalance},\mathcal G}(y\mid x,c)
    \propto
    \piref(y\mid x)
    \exp\!\left(
      \frac{E_{\mathrm{FlowBalance},\mathcal G}(y\mid x,c)}{\tau}
    \right),
    \label{eq:intro-target}
\end{equation}
and fits it through the profiled trajectory-balance residual
\begin{equation}
    \Delta_{\mathrm{TB}}(y;x,c,\mathcal G)
    =\tau\log Z_{\mathrm{FlowBalance},\mathcal G}(x,c)
     +\tau\log\frac{\pith(y\mid x)}{\piref(y\mid x)}
     -E_{\mathrm{FlowBalance},\mathcal G}(y\mid x,c).
    \label{eq:intro-tb}
\end{equation}
One log-partition estimate is profiled per rollout group. The reference policy controls support and drift, the verifier determines outcome direction, privileged hindsight provides dense within-trajectory evidence, and trajectory balance converts their composite energy into a normalized distribution over complete responses.

This construction can also be viewed as a \emph{distributional inner-loop update}: a task source supplies prompts, the current policy supplies experience, FlowBalance maps $\{(y_i,R_i,G_{\mathrm H,i})\}_{i=1}^N$ to a target distribution, and policy fitting produces the next policy. Our experiments deliberately hold the task distribution fixed, isolating this experience-to-policy update from outer-loop task generation or curriculum evolution. The latter is complementary rather than assumed by the method.

Our analysis studies the induced target directly. We show that profiling the partition preserves every within-group probability contrast; characterize FlowBalance as the unique minimum reverse-KL displacement from the reference at its attained composite-energy level; establish monotonic verifier control of the target reward statistic; and quantify exactly how outcome calibration converts false-positive self-guidance on a rejected response into a probability-ratio correction favoring a verified response. These are properties of the target and its local fitting objective, not a claim of global convergence for an arbitrary neural optimizer.

On mathematical reasoning, FlowBalance obtains the strongest overall averages on both Qwen3-4B and Qwen3-8B among GRPO, OPSD, RLSD, FlowRL, and FlowBalance. Relative to FlowRL, it improves the core four-benchmark average over AIME24, HMMT25, MATH500, and OlympiadBench by $1.67$ points on Qwen3-4B and $1.98$ points on Qwen3-8B, while also improving the full aggregate average. On Qwen3-8B, FlowBalance reaches $0.5$ AIME24 validation accuracy in about $100$ steps versus roughly $143$ for GRPO, remains stable over $400$ steps, and avoids the response-length collapse observed under direct OPSD. A controlled AIME24 diagnostic further finds higher correct-only semantic strategy diversity than GRPO and RLSD.

Our contributions are:
\begin{itemize}[leftmargin=1.5em,itemsep=0.9em]
    \item \textbf{A distributional self-improvement objective that improves over both verifier-only RL and RL--self-guidance hybrids.}
    GRPO improves reasoning from sparse verifier outcomes, while RLSD augments verifier-based policy optimization with dense same-model guidance. FlowBalance addresses a different question: \emph{what normalized complete-response distribution should the next policy learn from these signals?} It combines verifier advantages and privileged-hindsight guidance into one reference-supported trajectory energy, then fits the induced distribution through profiled trajectory balance. This makes the comparison to GRPO and RLSD especially informative: FlowBalance is not merely adding more supervision to GRPO or inserting another guidance term into an RL objective; it changes the policy-update object from a local optimization signal into an explicitly normalized distribution over the model's own reasoning trajectories.

    \emph{Concrete evidence.} On the core four-benchmark average over AIME24, HMMT25, MATH500, and OlympiadBench, FlowBalance improves over GRPO by $2.60$ points on Qwen3-4B ($67.69$ versus $65.10$) and $2.44$ points on Qwen3-8B ($71.09$ versus $68.65$). Relative to RLSD, the corresponding gains are $5.83$ points on Qwen3-4B ($67.69$ versus $61.87$) and $4.18$ points on Qwen3-8B ($71.09$ versus $66.92$). The same ordering holds on the full five-benchmark aggregate: FlowBalance exceeds GRPO by $1.95$ and $2.12$ points and RLSD by $4.71$ and $3.49$ points on the 4B and 8B backbones, respectively. On Qwen3-8B, it also reaches $0.5$ AIME24 validation accuracy in about $100$ steps rather than roughly $143$ for GRPO and remains near its peak over $400$ steps while GRPO degrades after about step $180$. These results isolate the practical contribution: learning a verifier-grounded trajectory distribution improves over both outcome-only RL and a direct RL--dense-guidance combination across two model scales.

    \item \textbf{Outcome-calibrated self-guidance that converts false confidence into a verifier-grounded correction.}
    The privileged-hindsight view is deliberately treated as \emph{useful but imperfect}: it can recognize promising reasoning structure, yet it can also assign positive confidence to a trajectory that fails the verifier. FlowBalance therefore uses $+\GH$ when the group-relative advantage is positive, $-\GH$ when it is negative, and zero guidance when the group provides no outcome preference. This is more than a heuristic filter. Proposition~\ref{prop:sign-gate-pairwise} shows that, for a verified success and a rejected response, sign gating multiplies their target probability ratio by the exact factor $\exp(2\beta_G\GH(y_-)/\tau)$ relative to ungated guidance. Dense same-model evidence can therefore refine verified successes without allowing a confidently scored failure to become self-reinforcing supervision.

    \emph{Concrete evidence.} In the exact four-mode diagnostic, reward-only shaping reaches success mass $0.818$, ungated self-guidance reaches $0.832$, and FlowBalance reaches $0.900$ while also increasing the robust-success mode from $0.327$ under reward-only shaping to $0.440$. In the binary false-confidence sweep at $G_-=0.5$, FlowBalance attains target success probability $0.894$, compared with $0.817$ for reward-only shaping and $0.807$ for ungated shaping. The language-model results show the same qualitative distinction: direct OPSD rapidly collapses to short responses, whereas FlowBalance maintains long reasoning traces; moreover, increasing the guidance coefficient from $\beta_G=1$ to $3$ lowers AIME24 from $89.33$ to $86.00$ and HMMT25 from $34.67$ to $30.00$. The benefit therefore comes from \emph{calibrated} self-guidance, not from simply making the dense signal stronger.

    \item \textbf{A conservative and information-efficient trajectory-balance update with exact target-level guarantees.}
    FlowBalance profiles one scalar log-partition value per rollout group, but this nuisance parameter removes only the common offset: all $N-1$ independent within-group probability contrasts remain available to train the policy. The induced target is also the unique minimum reverse-KL displacement from the reference among group distributions reaching its expected composite-energy level, and increasing the verifier coefficient monotonically increases the target's expected verifier reward for fixed guidance. Together, these results identify the update as a principled minimum-change self-improvement step that uses the full relative structure of the rollout group while keeping verified reward as an explicit control variable.

    \emph{Concrete evidence.} The exact synthetic diagnostics quantify both effects. A matched-energy full-support alternative requires reverse KL $0.973$, whereas the FlowBalance exponential tilt requires only $0.273$; the alternative is therefore $3.6\times$ farther from the reference. At rollout-group size $N=32$, profiling all contrasts yields $2.92\%$ of the local Gaussian parameter risk of a one-contrast-per-group estimator---about $1/34$ as much risk. Consistent with this conservative, all-contrast update---without claiming that the target-level results alone prove optimizer convergence---FlowBalance reaches $0.5$ AIME24 validation accuracy in about $100$ steps rather than roughly $143$ for GRPO and remains near its peak over $400$ steps, while GRPO degrades sharply after approximately step $180$.

    \item \textbf{Performance and semantic-diversity evidence that self-improvement need not collapse to one reasoning mode.}
    FlowBalance improves mathematical reasoning while preserving a broader successful response distribution. On Qwen3-4B, it improves over FlowRL on all four core reported benchmarks: AIME24, HMMT25, MATH500, and OlympiadBench. On Qwen3-8B, it obtains the best mean on every benchmark in the main table and improves the core four-benchmark average by $1.98$ points over FlowRL. These results matter for the paper's central claim because the largest gains are not confined to a single metric: they span both multi-sample competition reasoning and single-sample accuracy, while the training traces show that the policy does not obtain them by collapsing to the short-response behavior seen under direct OPSD.

    \emph{Concrete evidence.} In the controlled AIME24 diagnostic, correct-only Simpson strategy diversity is $0.2194$ for FlowBalance, compared with $0.1017$ for GRPO and $0.1456$ for RLSD. The associated traces differ in their mathematical representation, not merely their wording: FlowBalance discovers a hidden $4\!\times\!5\!\times\!8$ box embedding where GRPO uses Cayley--Menger, and the appendix shows envelope versus multiple-root reasoning, meridian-section versus implicit-normal geometry, and coordinate elimination versus secant--tangent hyperbola parameterization. Although this is a one-seed LLM-judged diagnostic rather than a population-level diversity guarantee, it supplies concrete qualitative evidence that the learned distribution places mass on genuinely different correct derivations instead of stylistic rewrites of one dominant template.
\end{itemize}

%% file: sections/data_mixture.tex
\section{Preliminaries}
\label{sec:preliminaries}

\subsection{On-Policy Reasoning Experience}

Let $x\sim\D$ be a reasoning prompt. A response is a variable-length token sequence $y=(y_1,\ldots,y_T)\in\Y(x)$. At token $t$, the state is $s_t=(x,y_{<t})$, and the trainable policy factorizes as
\begin{equation}
    \pith(y\mid x)
    =
    \prod_{t=1}^{T}\pith(y_t\mid s_t).
\end{equation}
At the start of each training iteration, we snapshot the current parameters as $\theta^-\leftarrow\theta$. The frozen rollout policy is $\pold$, while the reference policy $\piref$ is a fixed copy of the initial checkpoint.

For each prompt, we sample a group of responses $\{y^{(i)}\}_{i=1}^{N}\sim\pold(\cdot\mid x)$. The verifier assigns the terminal reward $R_i=R(y^{(i)};x)$ according to final-answer correctness. We compute the stopped group-relative advantage
\begin{equation}
    A_i
    =
    \frac{R_i-\mu_R(x)}
         {\sigma_R(x)+\epsilon},
    \label{eq:group-advantage}
\end{equation}
where $\mu_R(x)$ and $\sigma_R(x)$ are the mean and standard deviation of the rewards in the response group. Rewards, group statistics, sampled trajectories, and advantages receive no gradient.

\subsection{Privileged-Hindsight Self-Guidance}

Each training prompt $x$ is paired with training-only context $c$, such as a reference solution or task feedback. We evaluate the sampled token $y_t$ through two views of the same frozen snapshot:
\begin{align}
    \pi_{\mathrm{roll}}(\cdot\mid s_t)
    &=\pi_{\theta^-}(\cdot\mid x,y_{<t}),
    \label{eq:rollout-policy-view}\\
    \piH(\cdot\mid s_t,c)
    &=\pi_{\theta^-}(\cdot\mid x,c,y_{<t}).
    \label{eq:hindsight-policy-view}
\end{align}
The rollout view generates the response without observing $c$. The privileged-hindsight view sees $c$ only after the response has been sampled and scores the same tokens; it generates no replacement trajectory and receives no gradient. At inference, the deployed policy observes neither $c$ nor the hindsight view.

This same-model scoring construction is related to on-policy self-distillation, but its role in FlowBalance is narrower: it supplies a stopped dense feature for defining a trajectory-level target. FlowBalance does not optimize a separate token-level imitation loss.

\subsection{Reference-Supported Target Distribution}

We consider a normalized target over complete responses,
\begin{equation}
    p^\star(y\mid x,c)
    =
    \frac{1}{Z(x,c)}
    \piref(y\mid x)
    \exp\!\left(\frac{E(y\mid x,c)}{\tau}\right),
    \label{eq:generic-target}
\end{equation}
where $E(y\mid x,c)$ is a stopped trajectory energy, $\tau>0$ is its temperature, and $Z(x,c)$ is the partition function. The reference policy fixes support and controls drift; the energy specifies relative preference among complete responses. Section~\ref{sec:method} defines FlowBalance's outcome-calibrated self-guidance energy and fits the target through trajectory balance.

%% file: sections/method.tex
\section{FlowBalance: Verifier-Grounded Self-Improvement}
\label{sec:method}

FlowBalance maps a rollout group of self-generated reasoning trajectories into a normalized next-policy target. The verifier supplies sparse but grounded outcome evidence, while a privileged-hindsight view of the same frozen policy supplies dense self-guidance on the sampled tokens. FlowBalance first combines these stopped signals into an outcome-calibrated trajectory energy and then fits the induced complete-response distribution through trajectory balance. The dense branch therefore shapes \emph{which distribution is learned}; it is not optimized as a separate token-level imitation objective.

\subsection{Training-Time Self-Guidance from Privileged Hindsight}

For a sampled response $y$, define the clipped token-level hindsight gain
\begin{equation}
    \delta_t^{\mathrm H}(y;x,c)
    =
    \clip\left(
      \log\piH(y_t\mid s_t,c)
      -\log\piref(y_t\mid s_t),
      -B,B
    \right).
    \label{eq:token-guidance-score}
\end{equation}
We aggregate these gains over the complete response:
\begin{equation}
    \GH(y\mid x,c)
    =
    \frac{1}{T}
    \sum_{t=1}^{T}\delta_t^{\mathrm H}(y;x,c).
    \label{eq:trajectory-guidance-score}
\end{equation}
$\GH$ measures how much the privileged-hindsight view raises or lowers the average sampled-token log probability relative to the fixed reference policy. It is a trajectory feature computed on the policy's own on-policy experience. No token is resampled from $\piH$, and no gradient is propagated through $\piH$ or $\GH$.

\subsection{Outcome-Calibrated Self-Improvement Target}

Dense self-guidance can be useful while still being wrong about the verified outcome. FlowBalance therefore uses the group-relative advantage to determine its direction:
\begin{equation}
    E_{\mathrm{FlowBalance}}(y\mid x,c)
    =
    \eta_A A(y)
    +\beta_G\GH(y\mid x,c)\sign(A(y)),
    \label{eq:flow-energy}
\end{equation}
where $\eta_A,\beta_G\geq0$. If $A(y)>0$, positive guidance increases the trajectory energy. If $A(y)<0$, positive guidance is reversed, preventing a confidently scored failure from becoming self-reinforcing supervision. If $A(y)=0$, the dense branch is disabled. All quantities on the right-hand side are stopped during the policy update.

FlowBalance defines the unnormalized and normalized targets
\begin{align}
    \Rtil_{\mathrm{FlowBalance}}(y\mid x,c)
    &=
    \piref(y\mid x)
    \exp\!\left(\frac{E_{\mathrm{FlowBalance}}(y\mid x,c)}{\tau}\right),
    \label{eq:unnormalized-target}\\
    p^\star_{\mathrm{FlowBalance}}(y\mid x,c)
    &=
    \frac{\Rtil_{\mathrm{FlowBalance}}(y\mid x,c)}
         {Z_{\mathrm{FlowBalance}}(x,c)},
    \quad
    Z_{\mathrm{FlowBalance}}(x,c)=\sum_{y\in\Y(x)}\Rtil_{\mathrm{FlowBalance}}(y\mid x,c).
    \label{eq:normalized-target}
\end{align}
The three factors have distinct roles: $\piref$ retains reference support, $A$ supplies verified outcome direction, and $\GH$ provides dense within-trajectory evidence. Because $A$ is group-relative in implementation, the practical update uses the realized rollout group $\mathcal G$:
\begin{equation}
    p^\star_{\mathrm{FlowBalance},\mathcal G}(y^{(i)}\mid x,c)
    =
    \frac{\piref(y^{(i)}\mid x)\exp(E_{\mathrm{FlowBalance}}(y^{(i)}\mid x,c)/\tau)}
         {\sum_{j=1}^{N}\piref(y^{(j)}\mid x)\exp(E_{\mathrm{FlowBalance}}(y^{(j)}\mid x,c)/\tau)}.
    \label{eq:group-restricted-target}
\end{equation}
The response-space notation describes the corresponding ideal Gibbs distribution; the objective and profiled partition are evaluated on sampled complete trajectories.

\paragraph{Why a partition-normalized target?}
A local dense score specifies how individual sampled tokens look under privileged hindsight, but not where total probability mass should settle across complete responses. The partition term converts relative trajectory energies into a probability-conserving target. Consequently, self-guidance can refine the distribution within the verifier's direction without becoming an independent local loss. Appendix~\ref{app:detailed-balance} gives the corresponding equilibrium interpretation.

\subsection{Trajectory Balance as a Normalized Self-Update}
\label{sec:trajectory-balance}

The complete-trajectory balance equation is
\begin{equation}
    \tau\log Z_{\mathrm{FlowBalance}}(x,c)
    +\tau\log\frac{\pith(y\mid x)}{\piref(y\mid x)}
    -E_{\mathrm{FlowBalance}}(y\mid x,c)
    =0.
    \label{eq:complete-balance}
\end{equation}
The trajectory-balance residual is
\begin{equation}
    \Delta_{\mathrm{TB}}(y^{(i)};x,c)
    =
    \tau\log Z_{\mathrm{FlowBalance}}(x,c)
    +\tau\log\frac{\pith(y^{(i)}\mid x)}{\piref(y^{(i)}\mid x)}
    -E_{\mathrm{FlowBalance}}(y^{(i)}\mid x,c).
    \label{eq:tb-residual}
\end{equation}
At zero residual, the partition cancels in pairwise probability ratios:
\begin{equation}
    \log\frac{\pith(y^{(a)}\mid x)}{\pith(y^{(b)}\mid x)}
    =
    \log\frac{\piref(y^{(a)}\mid x)}{\piref(y^{(b)}\mid x)}
    +
    \frac{E_{\mathrm{FlowBalance}}(y^{(a)}\mid x,c)-E_{\mathrm{FlowBalance}}(y^{(b)}\mid x,c)}{\tau}.
    \label{eq:relative-mass}
\end{equation}
Thus the energy controls relative preference among the model's own sampled experiences, while $\log Z_{\mathrm{FlowBalance}}$ absorbs only the common prompt-level offset.

The main loss is
\begin{equation}
    \mathcal L_{\mathrm{FlowBalance}}(\theta)
    =
    \E_{(x,c)\sim\D}
    \left[
      \frac{1}{2N}\sum_{i=1}^{N}
      \Delta_{\mathrm{TB}}(y^{(i)};x,c)^2
    \right].
    \label{eq:FlowBalance-loss}
\end{equation}
Rewards, advantages, self-guidance scores, partition estimates, and sampled responses are stopped. Gradients are taken only through the trainable-policy log probabilities $\log\pith(y^{(i)}\mid x)$.

\paragraph{Subtrajectory balance.}\label{sec:subtrajectory-balance}
The same principle can be applied between intermediate states. Let $Z_{\mathrm{FlowBalance}}(s)$ denote the continuation partition from state $s$, with $Z_{\mathrm{FlowBalance}}(s_{\mathrm{term}})=1$. Define the per-token shaped increment
\begin{equation}
    r_t^{\mathrm{FlowBalance}}
    =
    \tau\log\piref(y_t\mid s_t)
    +\frac{\beta_G}{T}\delta_t^{\mathrm H}\sign(A)
    +\eta_AA\,\mathbf 1\{t=T\}.
    \label{eq:subtrajectory-increment}
\end{equation}
For an interval $i{:}j$, the subtrajectory residual is
\begin{equation}
    \Delta_{i:j}
    =
    \tau\log Z_{\mathrm{FlowBalance}}(s_i)
    +\tau\sum_{t=i}^{j-1}\log\pith(y_t\mid s_t)
    -\tau\log Z_{\mathrm{FlowBalance}}(s_j)
    -\sum_{t=i}^{j-1}r_t^{\mathrm{FlowBalance}}.
    \label{eq:subtrajectory-balance}
\end{equation}
The endpoint partitions account for different continuation distributions. Sampling intervals and minimizing $\E[\Delta_{i:j}^2]$ provides a denser fitting objective along long responses; the experiments in this paper use the complete-response implementation unless otherwise stated.

\subsection{Profiled Optimization over Rollout Groups}
\label{sec:training-algorithm}

The partition can be represented by a learned prompt-conditioned estimator or profiled directly from the rollout group. We use the group estimator. Each response implies
\begin{equation}
    \widehat{\log Z}_{i}(x,c)
    =
    \frac{E_{\mathrm{FlowBalance}}(y^{(i)}\mid x,c)}{\tau}
    -\log\frac{\pith(y^{(i)}\mid x)}{\piref(y^{(i)}\mid x)},
    \label{eq:per-trajectory-logz}
\end{equation}
and we set
\begin{equation}
    \widehat{\log Z}_{\mathrm{FlowBalance}}(x,c)
    =
    \frac{1}{N}\sum_{i=1}^{N}\widehat{\log Z}_{i}(x,c).
    \label{eq:group-logz}
\end{equation}
Gradients are stopped through this estimate.

\begin{algorithm}[t]
\caption{FlowBalance: verifier-grounded self-improvement from on-policy experience}
\label{alg:FlowBalance}
\begin{algorithmic}[1]
\Require Prompt--context dataset $\D$; policy $\pi_\theta$; fixed reference $\piref$; verifier $R$; coefficients $\eta_A,\beta_G,\tau$.
\For{each training iteration}
    \State Snapshot the current policy: $\theta^-\leftarrow\theta$.
    \State Sample a minibatch $\{(x_b,c_b)\}_{b=1}^{B}\sim\D$.
    \State Generate $N$ responses $y_b^{(i)}\sim\pi_{\theta^-}(\cdot\mid x_b)$ for every prompt.
    \State Evaluate rewards and compute stopped group-relative advantages $A_{b,i}$ using Eq.~\eqref{eq:group-advantage}.
    \State Score sampled tokens under $\pi_\theta$, $\piref$, and the frozen hindsight view $\piH=\pi_{\theta^-}(\cdot\mid x_b,c_b,y_{<t})$.
    \State Compute $\GH(y_b^{(i)}\mid x_b,c_b)$ and $E_{\mathrm{FlowBalance}}(y_b^{(i)})$ using Eqs.~\eqref{eq:trajectory-guidance-score} and~\eqref{eq:flow-energy}.
    \State Profile $\widehat{\log Z}_{\mathrm{FlowBalance}}(x_b,c_b)$ by Eq.~\eqref{eq:group-logz}.
    \State Compute $\Delta_{\mathrm{TB}}(y_b^{(i)};x_b,c_b)$ and update $\theta$ with Eq.~\eqref{eq:FlowBalance-loss}.
\EndFor
\end{algorithmic}
\end{algorithm}

Algorithm~\ref{alg:FlowBalance} makes the self-improvement loop explicit. The frozen snapshot first generates experience without privileged context, then provides a training-only hindsight score on those same trajectories. The verifier grounds the direction of the score, and trajectory balance internalizes the resulting normalized distribution into the next policy snapshot.

%% file: sections/theory.tex
\section{Why FlowBalance Supports Verifier-Grounded Self-Improvement}
\label{sec:theory}
\label{sec:sign-gated-theory}

Section~\ref{sec:method} defines FlowBalance as a partition-normalized update over a realized group of the model's own trajectories. We analyze four properties that are useful for a self-improvement inner loop. First, profiling the group partition should retain the relative evidence contained in all sampled responses. Second, the target should move conservatively from the reference. Third, the verifier should remain an explicit control on target reward even in the presence of dense self-guidance. Fourth, false-positive guidance on a rejected response should not become self-reinforcement. All statements below are conditional on the realized rollout group and treat verifier and guidance quantities as stopped. Proofs are in Appendix~\ref{app:sign-gated-proofs}.

%% file: sections/theory_sign_gated.tex
\subsection{Distributional Structure of the Self-Update}
\label{subsec:theory-target-properties}

For each sampled response $y^{(i)}$, let
\begin{equation}
    E_i
    =\eta_A A_i+\beta_G\GH(y^{(i)}\mid x,c)\sign(A_i)
\end{equation}
be its stopped FlowBalance energy. The normalized group target is
\begin{equation}
    p_i^\star
    =\frac{\piref(y^{(i)}\mid x)\exp(E_i/\tau)}{\sum_{j=1}^N\piref(y^{(j)}\mid x)\exp(E_j/\tau)}.
    \label{eq:theory-main-group-target}
\end{equation}
Let $p^\star=(p_1^\star,\ldots,p_N^\star)$. All results are conditional on the realized rollout group.

\begin{proposition}[Profiled balance uses all within-group contrasts]
\label{thm:profiled-balance}
For any candidate group distribution with positive mass on the realized responses, the profiled trajectory-balance loss is zero if and only if every relative probability contrast matches
\begin{equation}
    \frac{\pi(y^{(i)})}{\pi(y^{(j)})}
    =
    \frac{\piref(y^{(i)}\mid x)}{\piref(y^{(j)}\mid x)}
    \exp\!\left(\frac{E_i-E_j}{\tau}\right),
    \qquad i,j\in[N].
    \label{eq:main-profiled-log-ratio}
\end{equation}
In particular, whenever the group target is representable, the global minimum is zero; the unknown prompt-level normalizer removes exactly one common offset and preserves the remaining $N-1$ contrast directions.
\end{proposition}

The group partition therefore does not collapse the rollout group to one comparison. It absorbs only the common energy shift while retaining all relative evidence supplied by the model's own sampled experiences.

\begin{proposition}[Conservative minimum-change self-update]
\label{thm:flowbalance-razor}
Let the reference be restricted and renormalized to the realized rollout group when computing KL. For any group distribution $p\in\Delta_N$,
\begin{equation}
    \KL(p\|\piref)
    =
    \KL(p^\star\|\piref)
    +
    \KL(p\|p^\star)
    +
    \frac{1}{\tau}
    \left(
        \sum_{i=1}^N p_iE_i
        -
        \sum_{i=1}^N p_i^\star E_i
    \right).
    \label{eq:main-kl-decomposition}
\end{equation}
Consequently, among all group distributions that attain at least the expected FlowBalance energy of $p^\star$, the FlowBalance target is the unique minimum-reverse-KL displacement from the reference.
\end{proposition}

FlowBalance is therefore the smallest reference-supported distributional tilt that reaches its composite self-improvement energy level. This is a property of the target distribution, not a claim that an arbitrary finite neural-network update exactly reaches the target in one optimizer step.

\subsection{Verifier Control against Self-Confirmation}
\label{subsec:theory-component-guarantees}

\begin{proposition}[Verifier weight monotonically improves target reward]
\label{thm:verifier-monotonicity}
Hold the self-guidance scores and the realized rollout group fixed. Increasing the verifier coefficient $\eta_A$ monotonically shifts the FlowBalance target toward higher-reward responses. In particular,
\begin{equation}
    \frac{\partial}{\partial\eta_A}
    \E_{p^\star}[R]
    =
    \frac{\Var_{p^\star}(R)}
         {\tau(\sigma_R(x)+\epsilon)}
    \geq0 .
    \label{eq:main-verifier-monotonicity}
\end{equation}
\end{proposition}

The verifier remains an explicit control knob even when dense self-guidance is present: for fixed guidance, increasing its weight cannot reduce the target's expected verifier reward.

\begin{proposition}[Outcome calibration corrects false-positive self-guidance]
\label{prop:sign-gate-pairwise}
Consider a verified success $y_+$ and a verifier-rejected response $y_-$ in the same mixed-outcome group. Relative to otherwise identical \emph{ungated} self-guidance, sign gating changes their target probability ratio according to
\begin{equation}
    \frac{p^\star_{\mathrm{gated}}(y_+)}{p^\star_{\mathrm{gated}}(y_-)}
    =
    \frac{p^\star_{\mathrm{ungated}}(y_+)}{p^\star_{\mathrm{ungated}}(y_-)}
    \exp\!\left(\frac{2\beta_G\GH(y_-\mid x,c)}{\tau}\right).
    \label{eq:main-sign-gate-prob-ratio}
\end{equation}
Hence, when $\GH(y_-\mid x,c)>0$, positive self-guidance on a rejected response is converted from self-reinforcement into a probability-ratio correction favoring the verified response.
\end{proposition}

This proposition captures the intended useful-but-imperfect regime: privileged hindsight can assign positive local likelihood both to correct reasoning and to plausible failures. FlowBalance preserves the ability of guidance to distinguish successful modes while reversing its false-positive support where the verifier rejects the response. Appendix~\ref{sec:sign-gated-numerics} provides exactly enumerable diagnostics for the same target, including a reliability sweep over imperfect self-guidance.

%% file: sections/experiment.tex
\section{Experiments}
\label{sec:experiments}

We evaluate whether FlowBalance supports reliable policy self-improvement from on-policy reasoning experience along four axes. First, we compare final mathematical reasoning accuracy against a verifier-only RL baseline, direct on-policy distillation, a verifier--distillation hybrid, and outcome-only trajectory balance. Second, we inspect update efficiency, late-training stability, and response-length behavior. Third, we isolate verifier grounding and self-guidance strength. Fourth, we test whether the learned response distribution retains multiple successful strategies through an LLM-judged AIME24 diagnostic. Appendix~\ref{app:reasoning-experimental-details} gives the training and decoding details, Appendix~\ref{app:aime24-multisample-details} gives the diversity protocol, and Appendix~\ref{app:additional-results} gives the full ablation tables.

\subsection{Self-Improvement Evaluation Setup}
\label{sec:experimental-setup}

We evaluate Qwen3-4B and Qwen3-8B~\citep{yang2025qwen3}. Every deployed policy receives only the problem statement. For FlowBalance, a frozen current-policy snapshot additionally sees the training solution or task feedback only while scoring already sampled tokens; the resulting self-guidance is unavailable at inference. RLSD uses an analogous frozen current-policy scoring path, whereas OPSD uses a fixed privileged teacher according to its original formulation. All methods share the same training prompts, rollout group size, verifier, response-length cap, checkpoint schedule, and evaluation script within each backbone. We compare GRPO~\citep{shao2024deepseekmath,yu2025dapo}, OPSD~\citep{zhao2026opsd}, RLSD~\citep{yang2026rlsd}, FlowRL~\citep{zhu2025flowrl}, and FlowBalance. Table~\ref{tab:qwen-main} reports final benchmark accuracy, and Figure~\ref{fig:math-training-dynamics} reports the intermediate traces used to diagnose update efficiency, stability, and response-length behavior. Appendix~\ref{app:reasoning-experimental-details}, especially Table~\ref{tab:reasoning-config-fields}, lists the concrete training and decoding fields used for reproducibility.

\subsection{Verifier-Grounded Self-Improvement on Mathematical Reasoning}
\label{sec:math-reasoning}

\paragraph{Main Results.}
Table~\ref{tab:qwen-main} summarizes mathematical reasoning results across two model backbones. All entries report step-180 results over five seeds. AIME24 is evaluated with Pass@16 to reflect competition-style sampling, while HMMT25, Minerva, MATH500, and OlympiadBench use Pass@1. The reported aggregate averages summarize broad performance rather than tuning to a single benchmark.

\begin{table*}[t]
\centering
\tiny
\setlength{\tabcolsep}{2.5pt}
\caption{Mathematical reasoning results across Qwen3-4B and Qwen3-8B. AIME24 uses Pass@16; all other benchmarks use Pass@1. Entries are percentages reported as mean $\pm$ sample standard deviation over five seeds at step 180. ``Avg.'' averages the five reported benchmark means.}
\label{tab:qwen-main}
\resizebox{\textwidth}{!}{%
\begin{tabular}{ll c cccc c}
\toprule
Model & Method & AIME24@16 & HMMT25@1 & Minerva@1 & MATH500@1 & Olympiad@1 & Avg. \\
\midrule
\multirow{5}{*}{Qwen3-4B}
& GRPO
& $78.00 \pm 1.83$
& $26.67 \pm 2.36$ & $51.18 \pm 1.36$ & $92.04 \pm 0.98$ & $63.68 \pm 0.58$ & $62.31$ \\
& OPSD
& $65.33 \pm 3.80$
& $14.67 \pm 2.98$ & $47.28 \pm 0.56$ & $87.56 \pm 1.34$ & $55.76 \pm 1.47$ & $54.12$ \\
& RLSD
& $73.33 \pm 2.36$
& $21.33 \pm 3.80$ & $50.29 \pm 0.88$ & $91.44 \pm 0.52$ & $61.36 \pm 0.66$ & $59.55$ \\
& FlowRL
& $75.33 \pm 1.83$
& $30.67 \pm 4.94$ & $\mathbf{51.99 \pm 1.51}$ & $92.84 \pm 0.62$ & $65.25 \pm 0.49$ & $63.22$ \\
& FlowBalance
& $\mathbf{80.00 \pm 0.00}$
& $\mathbf{32.00 \pm 2.98}$ & $50.51 \pm 0.56$ & $\mathbf{93.28 \pm 0.59}$ & $\mathbf{65.49 \pm 0.92}$ & $\mathbf{64.26}$ \\
\midrule
\multirow{5}{*}{Qwen3-8B}
& GRPO
& $85.33 \pm 1.83$
& $31.33 \pm 7.67$ & $52.87 \pm 1.02$ & $93.16 \pm 0.83$ & $64.78 \pm 0.62$ & $65.49$ \\
& OPSD
& $48.67 \pm 3.80$
& $4.00 \pm 3.65$ & $38.46 \pm 3.85$ & $74.56 \pm 4.81$ & $40.09 \pm 3.57$ & $41.16$ \\
& RLSD
& $82.67 \pm 3.65$
& $28.00 \pm 1.83$ & $52.94 \pm 1.38$ & $93.44 \pm 0.17$ & $63.56 \pm 1.19$ & $64.12$ \\
& FlowRL
& $86.67 \pm 0.00$
& $30.67 \pm 4.35$ & $52.79 \pm 1.37$ & $92.92 \pm 0.50$ & $66.20 \pm 1.05$ & $65.85$ \\
& FlowBalance
& $\mathbf{89.33 \pm 1.49}$
& $\mathbf{34.67 \pm 9.89}$ & $\mathbf{53.68 \pm 0.78}$ & $\mathbf{93.52 \pm 0.30}$ & $\mathbf{66.85 \pm 0.46}$ & $\mathbf{67.61}$ \\
\bottomrule
\end{tabular}%
}
\end{table*}

On Qwen3-4B, FlowBalance reaches a five-benchmark average of $64.26$, improving over GRPO by $1.95$ points, OPSD by $10.14$ points, RLSD by $4.71$ points, and FlowRL by $1.04$ points. It outperforms OPSD and RLSD on all five benchmarks and improves over GRPO on AIME24, HMMT25, MATH500, and OlympiadBench; FlowRL obtains the highest Minerva mean on this backbone. On Qwen3-8B, FlowBalance reaches $67.61$ and obtains the best mean on every reported benchmark, improving over GRPO by $2.12$ points, OPSD by $26.45$ points, RLSD by $3.49$ points, and FlowRL by $1.76$ points. The comparison to FlowRL isolates the value of self-guidance inside the same trajectory-distribution family, while the comparisons to GRPO, OPSD, and RLSD show that neither verifier-only optimization nor direct local self-imitation explains the aggregate gain. A plausible explanation for OPSD's poor reasoning performance is that imitating a solution-conditioned privileged teacher can suppress epistemic verbalization and shorten reasoning, weakening uncertainty-driven exploration and self-correction~\citep{kim2026degradation}; this is consistent with Figure~\ref{fig:math-response-length}. The largest FlowBalance gain over stronger baselines appears on AIME24 Pass@16, where allocating mass across multiple correct trajectories is especially useful. The smaller but consistent gains on MATH500 and OlympiadBench indicate that this distributional update does not trade broad Pass@1 accuracy for a sampling-heavy metric.

\paragraph{Training Dynamics.}
Figure~\ref{fig:math-training-dynamics} summarizes the dynamics of repeated policy improvement on Qwen3-8B. The accuracy panels use matched rollout and evaluation budgets for FlowBalance and GRPO, while the response-length panel compares FlowBalance with direct OPSD. First, FlowBalance reaches $0.5$ AIME24 validation accuracy in about $100$ steps, compared with roughly $143$ for GRPO, a $1.43\times$ reduction in updates to the threshold. Second, FlowBalance remains near its peak accuracy throughout $400$ steps, whereas GRPO degrades sharply after approximately step $180$. Third, FlowBalance maintains substantially longer reasoning trajectories than direct OPSD, which rapidly collapses to short responses. These observations support a narrow but useful claim: in this fixed-task inner loop, verifier-grounded self-guidance accelerates and stabilizes the update while avoiding the shortcut behavior of ungrounded local imitation.

\begin{figure*}[t]
    \centering
    \begin{subfigure}[t]{0.32\textwidth}
        \centering
        \includegraphics[width=\linewidth]{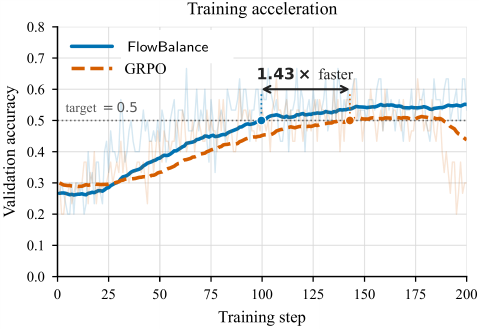}
        \caption{Training acceleration.}
        \label{fig:math-acceleration}
    \end{subfigure}
    \hfill
    \begin{subfigure}[t]{0.32\textwidth}
        \centering
        \includegraphics[width=\linewidth]{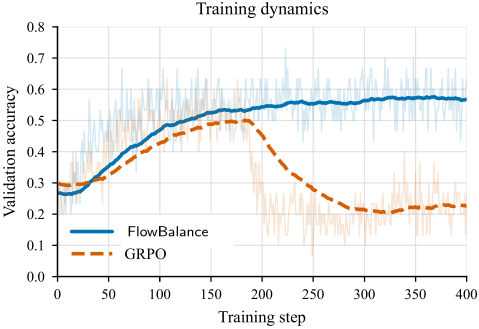}
        \caption{Training stability.}
        \label{fig:math-stability}
    \end{subfigure}
    \hfill
    \begin{subfigure}[t]{0.32\textwidth}
        \centering
        \includegraphics[width=\linewidth]{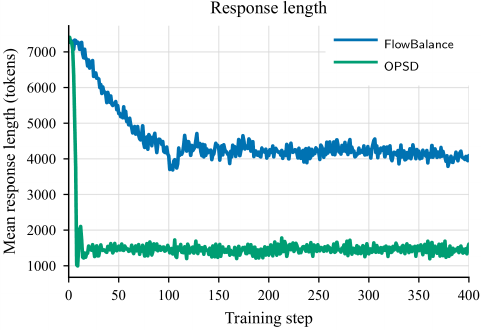}
        \caption{Response length.}
        \label{fig:math-response-length}
    \end{subfigure}
    \caption{\textbf{Training dynamics on mathematical reasoning with Qwen3-8B.} \textbf{(a) Training acceleration:} FlowBalance reaches $0.5$ AIME24 validation accuracy in about $100$ steps, compared with roughly $143$ steps for GRPO ($1.43\times$ faster). \textbf{(b) Training stability:} over $400$ training steps, FlowBalance remains near its peak performance, whereas GRPO degrades sharply after approximately step $180$. \textbf{(c) Response length:} direct OPSD rapidly collapses to substantially shorter responses, whereas FlowBalance maintains longer reasoning trajectories. In the accuracy panels, solid curves show smoothed trends and lighter curves show the corresponding per-step measurements.}
    \label{fig:math-training-dynamics}
\end{figure*}

\paragraph{Ablation Study.}
\label{sec:ablation-study}
We study the two coefficients in the FlowBalance energy:
\begin{equation}
    E_{\mathrm{FlowBalance}}(y)
    =
    \eta_A A_{\mathcal G}(y)
    +\beta_G\GH(y)\sign\!\big(A_{\mathcal G}(y)\big).
\end{equation}
We vary $\eta_A$ to control verifier grounding and $\beta_G$ to control self-guidance strength. The ablation uses one-dimensional sweeps around the default setting rather than a full grid, so that each table isolates one mechanism while holding the other coefficient fixed. Table~\ref{tab:ablation-eta-a} sweeps $\eta_A\in\{5,10,15\}$, and Table~\ref{tab:ablation-beta-q} sweeps the self-guidance coefficient $\beta_G\in\{1,2,3\}$. The default run corresponds to $\eta_A=15$ and $\beta_G=1$; the $\eta_A=10$ run and the full $\beta_G$ sweep are now complete. Each ablation row uses the same backbone, training budget, seeds, and five-benchmark average as Table~\ref{tab:qwen-main}; in addition to accuracy, we track training stability to detect settings that improve one benchmark at the cost of unstable optimization.

\begin{table}[t]
\centering
\begin{minipage}[t]{0.47\linewidth}
\centering
\small
\caption{Ablation over verifier coefficient $\eta_A$.}
\label{tab:ablation-eta-a}
\setlength{\tabcolsep}{5pt}
\begin{tabular}{lccc}
\toprule
$\eta_A$ & 5 & 10 & 15 \\
\midrule
Avg. & $65.65$ & $65.41$ & $67.61$ \\
\bottomrule
\end{tabular}
\end{minipage}
\hfill
\begin{minipage}[t]{0.47\linewidth}
\centering
\small
\caption{Ablation over self-guidance coefficient $\beta_G$.}
\label{tab:ablation-beta-q}
\setlength{\tabcolsep}{5pt}
\begin{tabular}{lccc}
\toprule
$\beta_G$ & 1 & 2 & 3 \\
\midrule
Avg. & $67.61$ & $66.48$ & $65.95$ \\
\bottomrule
\end{tabular}
\end{minipage}
\end{table}

\subsection{Correct-Strategy Diversity under Self-Improvement}
\label{sec:diversity-study}

\begin{wrapfigure}{r}{0.43\textwidth}
    \vspace{-1.0em}
    \centering
    \includegraphics[width=0.41\textwidth]{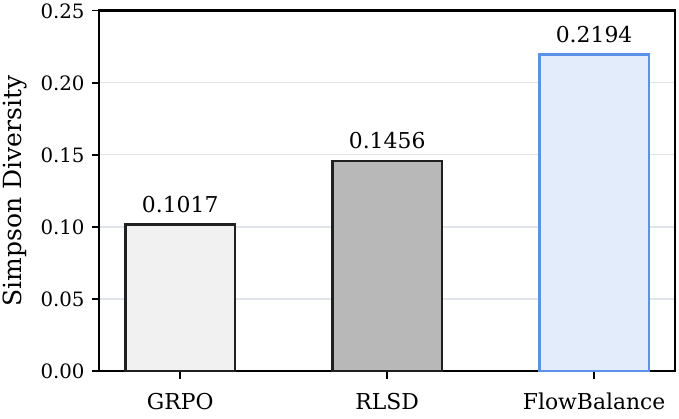}
    \caption{\textbf{LLM-judged strategy diversity.} Correct-only Simpson diversity on AIME24; protocol and scope are given in Appendix~\ref{app:aime24-multisample-details}.}
    \label{fig:llm-judged-diversity}
    \vspace{-0.8em}
\end{wrapfigure}

A reliable self-improvement update should not obtain accuracy only by sharpening one successful template. FlowBalance also broadens the kinds of correct solutions represented by the policy. We measure this with an LLM-judged semantic strategy diagnostic on AIME24: a GPT-5.5 judge extracts each full trajectory's mathematical representation and tools, clusters trajectories with the same core strategy, and we report Simpson strategy diversity,
\begin{equation}
    D_{\mathrm{Simpson}}=1-\sum_k p_k^2,
\end{equation}
where $p_k$ is the fraction of correct trajectories in strategy cluster $k$. This metric is the probability that two sampled correct trajectories use different semantic strategies; Appendix~\ref{app:aime24-multisample-details} gives the full protocol and scope.

Figure~\ref{fig:llm-judged-diversity} reports correct-only Simpson diversity. FlowBalance achieves $0.2194$, compared with $0.1017$ for GRPO and $0.1456$ for RLSD. Thus, within this diagnostic setting, FlowBalance's successful responses span a broader range of semantic solution strategies.

\begin{table*}[t]
\centering
\footnotesize
\setlength{\tabcolsep}{6pt}
\renewcommand{\arraystretch}{1.18}
\newcommand{\rollbox}[1]{\fcolorbox{black!35}{white}{\strut #1}}
\newcommand{\bluebox}[1]{\fcolorbox{blue!45!black}{white}{\strut #1}}
\caption{\textbf{Case study on AIME24 Problem 23.} Boxed spans indicate key reasoning actions in representative correct trajectories; ``$\cdots$'' denotes omitted intermediate text.}
\label{tab:aime24-case-study}
\begin{tabularx}{\textwidth}{p{0.11\textwidth}X}
\toprule
& \centering\textbf{Content (boxed = reasoning actions; ``$\cdots$'' = omitted)}\tabularnewline
\midrule
\textbf{Question}
&
Tetrahedron $ABCD$ satisfies
$AB=CD=\sqrt{41}$, $AC=BD=\sqrt{80}$, and $BC=AD=\sqrt{89}$.
If the common distance from the incenter to the four faces is $m\sqrt n/p$,
find $m+n+p$.
\tabularnewline
\midrule
\rowcolor{black!5}
\textbf{GRPO}
&
\textbf{Common Cayley--Menger route.}
``$\cdots$ all faces have sides $\sqrt{41},\sqrt{80},\sqrt{89}$ $\cdots$
\rollbox{$S=24\sqrt{21}$} $\cdots$
\rollbox{build Cayley--Menger matrix $M$ from $41,80,89$} $\cdots$
\rollbox{diagonalize the distance block by symmetry} $\cdots$
\rollbox{$\det M=819200$} $\cdots$
\rollbox{$V=160/3$} $\cdots$
\rollbox{$r=3V/S=20\sqrt{21}/63$} $\cdots$
\rollbox{$20+21+63=104$} $\cdots$''
\tabularnewline
\midrule
\rowcolor{blue!10}
\textbf{\textcolor{blue!65!black}{FlowBalance}}
&
\textbf{Hidden box-embedding route.}
``$\cdots$ observe
\bluebox{$41=4^2+5^2$, $80=4^2+8^2$, $89=5^2+8^2$} $\cdots$
\bluebox{embed in a $4\times5\times8$ box} $\cdots$
\bluebox{$A=(0,0,0)$, $B=(4,5,0)$, $C=(4,0,8)$, $D=(0,5,8)$} $\cdots$
\bluebox{verify all six edge lengths} $\cdots$
\bluebox{scalar triple product gives $V=160/3$} $\cdots$
\bluebox{$r=20\sqrt{21}/63$} $\cdots$
\bluebox{$20+21+63=104$} $\cdots$''
\tabularnewline
\bottomrule
\end{tabularx}
\end{table*}

\paragraph{Case Study.}
Table~\ref{tab:aime24-case-study} illustrates what the LLM-judged strategy clusters are intended to capture: differences in the mathematical objects and representations that make a solution work. We choose this example for its structural contrast rather than for per-problem win rate. GRPO uses a standard and broadly applicable Cayley--Menger determinant solution: it treats the tetrahedron as six pairwise distances and obtains the volume from a distance invariant. FlowBalance follows a different route by recognizing $41=4^2+5^2$, $80=4^2+8^2$, and $89=5^2+8^2$, which reveals a hidden $4\times5\times8$ rectangular-box embedding and reduces the volume computation to a scalar triple product. This case study therefore grounds the aggregate LLM-judged result: FlowBalance can place probability mass on correct trajectories that are not merely rephrasings of the dominant solution template, but use a different mathematical representation altogether. Appendix~\ref{app:additional-diversity-cases} gives three additional AIME24 examples in the same format, covering envelope versus multiple-root reasoning, planar circle tangency versus implicit-surface normals, and coordinate elimination versus trigonometric hyperbola parameterization.

%% file: sections/related_work.tex
\section{Related Work}
\label{sec:related-work}

\paragraph{Reinforcement learning from verifiable outcomes.}
RLVR is a major route to improving LLM reasoning with automatically checkable answers~\citep{shao2024deepseekmath,yu2025dapo,zhou2025evolving,yu2025rlev,liang2025g2rl,xia2026dirl}. Group-based methods estimate advantages from multiple responses to one prompt without a separately learned critic. Recent work improves this optimization through sequence-level ratios, trust regions, variance-aware weighting, dynamic clipping, and adaptive baselines or learning rates~\citep{zheng2025gspo,xie2025spo,zhang2025gvpo,yang2025dcpo,huang2026oblr}. These methods provide reliable outcome grounding, but their supervision remains response-level. FlowBalance addresses the complementary experience-to-policy question: how sparse outcomes and dense training-time evidence should define a normalized distribution over complete responses.

\paragraph{Privileged self-guidance and on-policy distillation.}
Knowledge distillation transfers behavior from a teacher to a student, while on-policy distillation evaluates student-generated trajectories under a teacher~\citep{bucila2006model,hinton2015distilling,agarwal2024onpolicy}. Recent methods use a privileged view of the same model, conditioned on reference answers, demonstrations, correct rollouts, or environment feedback~\citep{hubotter2026sdpo,zhao2026opsd,penaloza2026privileged,shenfeld2026sdft,yang2026rlsd,ye2026opcd}. This produces dense token-level information, but it can also shorten reasoning, suppress uncertainty, leak unavailable context into the learning signal, or plateau and degrade~\citep{yang2026rlsd,kim2026degradation,kim2026compresses,zhu2026manyfaces}. Existing work modifies, selects, or contrasts local teacher signals~\citep{pan2026rlcsd,yu2026dopd,zhao2026rosd,zhang2026dasd,jia2026aopd,xu2026tip,kim2026rebellious}; concurrent $\beta$-OPSD derives an efficient KL-regularized target with a tunable reference--teacher tradeoff~\citep{xu2026betaopsd}. FlowBalance uses the same-model privileged scorer only as a stopped \emph{self-guidance feature}. The algorithmic object is instead a verifier-grounded complete-response distribution fitted by trajectory balance, with no separate token-level imitation loss.

\paragraph{Guided self-improvement and self-evolving systems.}
Self-evolving systems add an outer loop that generates, selects, or schedules new learning experiences. R-Few, for example, stabilizes a Challenger--Solver loop with a small pool of human anchors and an online difficulty-based curriculum~\citep{yu2025rfew}. Its primary object is the evolving task and curriculum distribution. FlowBalance studies a complementary inner-loop object: given prompts and on-policy solution trajectories, how should those experiences update the policy without amplifying false self-confidence or collapsing successful-strategy diversity? We do not combine the two systems experimentally; rather, R-Few illustrates how an outer-loop experience generator could supply tasks to a FlowBalance-style distributional policy update.

\paragraph{Distribution matching and trajectory balance.}
GFlowNets learn stochastic construction policies that sample objects in proportion to an unnormalized reward rather than concentrating only on maximizers~\citep{bengio2021gflownet,bengio2023gflownet,xu2025robustllm,panaganti2026gdro,panaganti2022robust}. Trajectory balance enforces this condition over complete trajectories and improves long-horizon credit assignment~\citep{malkin2022trajectory}. For LLM reasoning, Flow of Reasoning and FlowRL instantiate trajectory balance with a prompt-conditioned partition term~\citep{yu2025for,zhu2025flowrl}, while GFlowRL replaces the auxiliary partition network with an in-batch rollout-group estimate~\citep{du2026gflowrl}. FlowBalance builds on this distributional lineage but changes the target energy: verified outcome advantages are combined with outcome-calibrated privileged-hindsight guidance, and the resulting reference-supported distribution is profiled over the rollout group. The comparison to FlowRL in our experiments isolates this additional guidance inside the same broad trajectory-balance family.

%% file: sections/conclusion.tex
\section{Conclusions and Discussions}
\label{sec:conclusion}

\paragraph{Conclusion.}
FlowBalance treats policy self-improvement as a distribution-learning problem over the model's own on-policy reasoning experience. A frozen privileged-hindsight view of the same policy supplies dense self-guidance, a verifier determines whether that guidance should be retained or reversed, and profiled trajectory balance fits the resulting reference-supported target over complete responses. This separation is central: dense guidance informs the target, but the policy is updated only through the normalized trajectory-balance objective. The theory shows that the profiled update preserves all within-group contrasts, is the minimum reverse-KL displacement at its attained energy level, retains monotonic verifier control of target reward, and converts false-positive guidance on rejected trajectories into an anti-self-confirmation correction. Across Qwen3-4B and Qwen3-8B, FlowBalance achieves the strongest five-benchmark average among GRPO, OPSD, RLSD, FlowRL, and FlowBalance. It also reaches the AIME24 validation threshold in fewer updates than GRPO, remains stable over extended training, avoids direct OPSD's response-length collapse, and exhibits higher correct-only semantic strategy diversity in the controlled AIME24 diagnostic.

\paragraph{Limitations and scope.}
Four limitations delimit these claims. First, the large-scale experiments focus on mathematical reasoning, so generalization to agentic, multimodal, or other long-horizon domains remains open. Second, the response-length comparison diagnoses a correlated failure mode but does not establish that longer responses cause higher accuracy. Third, the diversity analysis is an LLM-judged AIME24 diagnostic at one checkpoint and seed rather than a multi-seed human study. Fourth, we isolate the experience-to-policy inner loop on a fixed prompt distribution; FlowBalance is not by itself a full self-evolving system that generates or curates new tasks. A natural next step is to combine the distributional update with an outer-loop curriculum or guided task generator, including R-Few-style Challenger--Solver systems, while separately measuring task drift, policy stability, and strategy diversity.

%% file: sections/appendix.tex
\section{Detailed-Balance View of FlowBalance}
\label{app:detailed-balance}

\paragraph{Detailed balance.}
Let $q$ be a probability distribution on the response space $\Y(x)$, and let $K(y\to y')$ be a Markov transition kernel. The kernel is reversible with respect to $q$ if, for every pair $y,y'\in\Y(x)$,
\begin{equation}
    q(y)K(y\to y')
    =
    q(y')K(y'\to y).
    \label{eq:app-detailed-balance-kernel}
\end{equation}
This is the detailed-balance condition: at equilibrium, the probability flow from $y$ to $y'$ equals the reverse flow from $y'$ to $y$. Summing Eq.~\eqref{eq:app-detailed-balance-kernel} over $y$ shows that $q$ is a stationary distribution of $K$. Whenever both transition probabilities are positive, detailed balance also gives
\begin{equation}
    \frac{K(y\to y')}{K(y'\to y)}
    =
    \frac{q(y')}{q(y)}.
    \label{eq:app-detailed-balance-transition-ratio}
\end{equation}

\paragraph{FlowBalance as an equilibrium distribution.}
For a fixed prompt--context pair $(x,c)$, FlowBalance defines the Gibbs target
\begin{equation}
    p^\star(y\mid x,c)
    =
    \frac{1}{Z_{\mathrm{FlowBalance}}(x,c)}
    \piref(y\mid x)
    \exp\!\left(
      \frac{E_{\mathrm{FlowBalance}}(y\mid x,c)}{\tau}
    \right).
    \label{eq:app-detailed-balance-target}
\end{equation}
If a reversible Markov kernel were constructed with $p^\star$ as its equilibrium distribution, its forward--reverse transition ratio would satisfy
\begin{equation}
    \frac{K(y\to y')}{K(y'\to y)}
    =
    \frac{\piref(y'\mid x)}{\piref(y\mid x)}
    \exp\!\left(
      \frac{
        E_{\mathrm{FlowBalance}}(y'\mid x,c)
        -E_{\mathrm{FlowBalance}}(y\mid x,c)
      }{\tau}
    \right).
    \label{eq:app-flowbalance-transition-ratio}
\end{equation}
Thus the reference policy supplies the baseline relative mass, while the FlowBalance energy exponentially reweights transitions toward responses preferred by the verifier-grounded self-guidance energy. The partition function does not appear in this pairwise ratio, but it is required to turn these relative preferences into a normalized distribution.

\paragraph{Trajectory-balance interpretation.}
FlowBalance does not construct or simulate the Markov kernel $K$. Instead, trajectory balance directly imposes the global identity
\begin{equation}
    Z_{\mathrm{FlowBalance}}(x,c)\pith(y\mid x)
    =
    \piref(y\mid x)
    \exp\!\left(
      \frac{E_{\mathrm{FlowBalance}}(y\mid x,c)}{\tau}
    \right)
    \label{eq:app-trajectory-balance-identity}
\end{equation}
for each complete response. Taking logarithms and multiplying by $\tau$ gives the trajectory-balance residual
\begin{equation}
    \tau\log Z_{\mathrm{FlowBalance}}(x,c)
    +\tau\log\frac{\pith(y\mid x)}{\piref(y\mid x)}
    -E_{\mathrm{FlowBalance}}(y\mid x,c).
    \label{eq:app-tb-as-equilibrium-residual}
\end{equation}
When this residual is zero for all responses, $\pith$ equals the normalized target in Eq.~\eqref{eq:app-detailed-balance-target} and therefore has the same pairwise probability ratios as the equilibrium distribution associated with Eq.~\eqref{eq:app-flowbalance-transition-ratio}. In this sense, detailed balance provides an equilibrium interpretation, whereas trajectory balance is the learning constraint that fits the desired distribution without explicitly defining pairwise response transitions.

%% file: sections/appendix_sign_gated.tex
\section{Proofs for Theoretical Analyses}
\label{app:proofs}
\label{app:sign-gated-proofs}

Throughout this appendix we condition on the same fixed prompt--context pair and realized rollout group as in Section~\ref{sec:theory}. We use the main-text notation
\begin{equation}
    E_i=\eta_A A_i+\beta_G\GH(y^{(i)}\mid x,c)\sign(A_i),
    \qquad
    p_i^\star
    =\frac{\piref(y^{(i)}\mid x)\exp(E_i/\tau)}
    {\sum_{j=1}^N\piref(y^{(j)}\mid x)\exp(E_j/\tau)}.
\end{equation}
When KL is computed over the realized rollout group, \(\piref(\cdot\mid x)\) denotes its restriction and renormalization to that group.

\subsection{Distributional Properties}

\begin{appendixproposition}[Profiled balance matches all within-group contrasts]
\label{app:prop-profiled-balance}
For a candidate policy $\pi$ and scalar partition/intercept $z$, consider the realized-group trajectory-balance loss
\begin{equation}
    \mathcal L_N(\pi,z)
    =\frac{1}{2N}\sum_{i=1}^N
    \left[
        \tau z
        +\tau\log\frac{\pi(y^{(i)})}{\piref(y^{(i)}\mid x)}
        -E_i
    \right]^2.
    \label{eq:app-profiled-loss-statement}
\end{equation}
After profiling out $z$, the trajectory-balance loss is zero if and only if every within-group relative probability contrast matches
\begin{equation}
    \frac{\pi(y^{(i)})}{\pi(y^{(j)})}
    =
    \frac{\piref(y^{(i)}\mid x)}{\piref(y^{(j)}\mid x)}
    \exp\!\left(\frac{E_i-E_j}{\tau}\right),
    \qquad i,j\in[N].
    \label{eq:app-profiled-contrast-statement}
\end{equation}
Whenever the group target is representable, the global minimum is therefore zero. Profiling removes only one common group-level offset and preserves the remaining $N-1$ contrast directions.
\end{appendixproposition}

\begin{proof}
For a candidate policy \(\pi\), the trajectory-balance residual for response \(y^{(i)}\) and scalar partition/intercept \(z\) is
\begin{equation}
    \tau z
    +\tau\log\frac{\pi(y^{(i)})}{\piref(y^{(i)}\mid x)}
    -E_i .
\end{equation}
The empirical trajectory-balance loss over the realized group is therefore
\begin{equation}
    \mathcal L_N(\pi,z)
    =\frac{1}{2N}\sum_{i=1}^N
    \left[
        \tau z
        +\tau\log\frac{\pi(y^{(i)})}{\piref(y^{(i)}\mid x)}
        -E_i
    \right]^2 .
    \label{eq:app-profiled-tb-loss}
\end{equation}
For fixed \(\pi\), this is a convex quadratic function of the single scalar \(z\). Differentiating Eq.~\eqref{eq:app-profiled-tb-loss} with respect to \(z\) gives
\begin{equation}
    \frac{\partial \mathcal L_N(\pi,z)}{\partial z}
    =\frac{\tau}{N}\sum_{i=1}^N
    \left[
        \tau z
        +\tau\log\frac{\pi(y^{(i)})}{\piref(y^{(i)}\mid x)}
        -E_i
    \right].
\end{equation}
Thus the profiled intercept \(\widehat z\) is the unique solution of
\begin{equation}
    \sum_{i=1}^N
    \left[
        \tau \widehat z
        +\tau\log\frac{\pi(y^{(i)})}{\piref(y^{(i)}\mid x)}
        -E_i
    \right]
    =0 .
    \label{eq:app-profiled-z-foc}
\end{equation}
After profiling, the loss is zero if and only if every squared residual in Eq.~\eqref{eq:app-profiled-tb-loss} is zero at \(z=\widehat z\). Hence, for every pair \(i,j\),
\begin{align}
    \tau \widehat z
    +\tau\log\frac{\pi(y^{(i)})}{\piref(y^{(i)}\mid x)}
    -E_i
    &=0, \\
    \tau \widehat z
    +\tau\log\frac{\pi(y^{(j)})}{\piref(y^{(j)}\mid x)}
    -E_j
    &=0.
\end{align}
Subtracting the second equality from the first cancels the profiled scalar \(\tau\widehat z\) and yields
\begin{equation}
    \tau\log\frac{\pi(y^{(i)})}{\pi(y^{(j)})}
    -\tau\log\frac{\piref(y^{(i)}\mid x)}{\piref(y^{(j)}\mid x)}
    =E_i-E_j .
\end{equation}
Dividing by \(\tau\), exponentiating, and rearranging the reference-policy ratio gives Eq.~\eqref{eq:app-profiled-contrast-statement}.

Conversely, if Eq.~\eqref{eq:app-profiled-contrast-statement} holds for all pairs, then the quantities
\begin{equation}
    \tau\log\frac{\pi(y^{(i)})}{\piref(y^{(i)}\mid x)}-E_i
\end{equation}
are all equal to the same constant. Choosing \(z\) to be minus that constant divided by \(\tau\) makes every trajectory-balance residual zero. Therefore zero profiled loss is equivalent to matching all within-group target log probability-ratio contrasts. The scalar \(z\) removes exactly one common offset and cannot change any pairwise contrast.
\end{proof}

\begin{appendixproposition}[Minimum reference displacement]
\label{app:prop-flowbalance-razor}
Let $\piref$ denote the reference distribution restricted and renormalized to the realized rollout group. Then, for any group distribution $p\in\Delta_N$,
\begin{equation}
    \KL(p\|\piref)
    =
    \KL(p^\star\|\piref)
    +\KL(p\|p^\star)
    +\frac{1}{\tau}
    \left(
        \sum_{i=1}^N p_iE_i
        -\sum_{i=1}^N p_i^\star E_i
    \right).
    \label{eq:app-kl-decomposition-statement}
\end{equation}
Consequently, among all group distributions satisfying $\sum_i p_iE_i\geq\sum_i p_i^\star E_i$, the FlowBalance target $p^\star$ uniquely minimizes the reverse-KL displacement $\KL(p\|\piref)$.
\end{appendixproposition}

\begin{proof}
From Eq.~\eqref{eq:theory-main-group-target}, for each group element,
\begin{equation}
    \log\frac{p_i^\star}{\piref(y^{(i)}\mid x)}
    =\frac{E_i}{\tau}-\log\sum_{j=1}^N\piref(y^{(j)}\mid x)\exp(E_j/\tau).
\end{equation}
Let \(p\in\Delta_N\) be any group distribution. Then
\begin{align}
    \KL(p\|\piref)
    &=\sum_{i=1}^N p_i\log\frac{p_i}{p_i^\star}
      +\sum_{i=1}^N p_i\log\frac{p_i^\star}{\piref(y^{(i)}\mid x)} \\
    &=\KL(p\|p^\star)
      +\frac{1}{\tau}\sum_{i=1}^N p_iE_i
      -\log\sum_{j=1}^N\piref(y^{(j)}\mid x)\exp(E_j/\tau).
    \label{eq:app-kl-to-target-expand}
\end{align}
Applying the same identity with \(p=p^\star\) gives
\begin{equation}
    \KL(p^\star\|\piref)
    =\frac{1}{\tau}\sum_{i=1}^N p_i^\star E_i
      -\log\sum_{j=1}^N\piref(y^{(j)}\mid x)\exp(E_j/\tau).
\end{equation}
Subtracting this expression from Eq.~\eqref{eq:app-kl-to-target-expand} yields Eq.~\eqref{eq:app-kl-decomposition-statement}:
\begin{equation}
    \KL(p\|\piref)
    =
    \KL(p^\star\|\piref)
    +
    \KL(p\|p^\star)
    +
    \frac{1}{\tau}
    \left(
        \sum_{i=1}^N p_iE_i
        -
        \sum_{i=1}^N p_i^\star E_i
    \right).
\end{equation}
If \(p\) attains at least the expected FlowBalance energy of \(p^\star\), the last term is nonnegative; \(\KL(p\|p^\star)\ge0\) as well, with equality only when \(p=p^\star\). Hence \(p^\star\) is the unique minimum-reverse-KL displacement from the reference among all such distributions.
\end{proof}

\subsection{Effects of Verifier Weighting and Sign Gating}

\begin{appendixproposition}[Verifier weight monotonically improves reward statistics]
\label{app:prop-verifier-monotonicity}
Hold the self-guidance scores and the realized rollout group fixed. For the GRPO-normalized advantage
\begin{equation}
    A_i=\frac{R_i-\bar R}{\sigma_R(x)+\epsilon},
\end{equation}
increasing the verifier coefficient $\eta_A$ monotonically increases the expected reward under the FlowBalance target:
\begin{equation}
    \frac{\partial}{\partial\eta_A}\E_{p^\star}[R]
    =
    \frac{\Var_{p^\star}(R)}
         {\tau(\sigma_R(x)+\epsilon)}
    \geq0.
    \label{eq:app-verifier-monotonicity-statement}
\end{equation}
\end{appendixproposition}

\begin{proof}
Hold the self-guidance scores fixed. The only dependence of \(p_i^\star\) on \(\eta_A\) is through the term \(\eta_A A_i\). Differentiating the normalized exponential-family form gives
\begin{equation}
    \frac{\partial p_i^\star}{\partial\eta_A}
    =\frac{p_i^\star}{\tau}
    \left(A_i-\sum_{j=1}^N p_j^\star A_j\right).
\end{equation}
Therefore
\begin{equation}
    \frac{\partial}{\partial\eta_A}\E_{p^\star}[R]
    =\frac{1}{\tau}\Cov_{p^\star}(R,A).
\end{equation}
For the GRPO-normalized advantage used in the main text,
\begin{equation}
    A_i=\frac{R_i-\bar R}{\sigma_R(x)+\epsilon},
\end{equation}
where \(\bar R\) and \(\sigma_R(x)+\epsilon\) are fixed for the realized group. Hence
\begin{equation}
    \Cov_{p^\star}(R,A)
    =\frac{\Var_{p^\star}(R)}{\sigma_R(x)+\epsilon},
\end{equation}
and thus
\begin{equation}
    \frac{\partial}{\partial\eta_A}
    \E_{p^\star}[R]
    =
    \frac{\Var_{p^\star}(R)}
         {\tau(\sigma_R(x)+\epsilon)}
    \ge0,
\end{equation}
which proves Eq.~\eqref{eq:app-verifier-monotonicity-statement}.
\end{proof}

\begin{appendixproposition}[Sign gating corrects self-guidance support on failures]
\label{app:prop-sign-gate-pairwise}
Consider a verified success $y_+$ with positive advantage and a verifier-rejected response $y_-$ with negative advantage in the same rollout group. Relative to otherwise identical ungated self-guidance shaping, sign gating changes their target probability ratio according to
\begin{equation}
    \frac{p^\star_{\mathrm{gated}}(y_+)}{p^\star_{\mathrm{gated}}(y_-)}
    =
    \frac{p^\star_{\mathrm{ungated}}(y_+)}{p^\star_{\mathrm{ungated}}(y_-)}
    \exp\!\left(\frac{2\beta_G\GH(y_-\mid x,c)}{\tau}\right).
    \label{eq:app-sign-gate-ratio-statement}
\end{equation}
Hence positive privileged support on a verifier-rejected response is converted from self-reinforcement pressure into a probability-ratio correction favoring the verified response.
\end{appendixproposition}

\begin{proof}
Consider a verified success \(y_+\) and a verifier-rejected response \(y_-\) in the same group. Their probability ratio under the gated target is
\begin{equation}
    \frac{p^\star_{\mathrm{gated}}(y_+)}{p^\star_{\mathrm{gated}}(y_-)}
    =
    \frac{\piref(y_+\mid x)}{\piref(y_-\mid x)}
    \exp\!\left(
        \frac{E^{\mathrm{gated}}_+-E^{\mathrm{gated}}_-}{\tau}
    \right),
\end{equation}
where the group normalizer cancels. The corresponding ungated target has
\begin{equation}
    \frac{p^\star_{\mathrm{ungated}}(y_+)}{p^\star_{\mathrm{ungated}}(y_-)}
    =
    \frac{\piref(y_+\mid x)}{\piref(y_-\mid x)}
    \exp\!\left(
        \frac{E^{\mathrm{ungated}}_+-E^{\mathrm{ungated}}_-}{\tau}
    \right).
\end{equation}
The reference ratio cancels when comparing the gated and ungated probability ratios. Since \(y_+\) has positive advantage and \(y_-\) has negative advantage,
\begin{align}
    E^{\mathrm{gated}}_+-E^{\mathrm{gated}}_-
    &=\eta_A\left(A_{\mathcal G}(y_+)-A_{\mathcal G}(y_-)\right)
      +\beta_G\GH(y_+\mid x,c)+\beta_G\GH(y_-\mid x,c), \\
    E^{\mathrm{ungated}}_+-E^{\mathrm{ungated}}_-
    &=\eta_A\left(A_{\mathcal G}(y_+)-A_{\mathcal G}(y_-)\right)
      +\beta_G\GH(y_+\mid x,c)-\beta_G\GH(y_-\mid x,c).
\end{align}
The verifier difference \(\eta_A\left(A_{\mathcal G}(y_+)-A_{\mathcal G}(y_-)\right)\) is identical in the two expressions, and the self-guidance contribution differs by \(2\beta_G\GH(y_-\mid x,c)\). Therefore
\begin{equation}
    \frac{p^\star_{\mathrm{gated}}(y_+)}{p^\star_{\mathrm{gated}}(y_-)}
    =
    \frac{p^\star_{\mathrm{ungated}}(y_+)}{p^\star_{\mathrm{ungated}}(y_-)}
    \exp\!\left(\frac{2\beta_G\GH(y_-\mid x,c)}{\tau}\right),
\end{equation}
which proves Eq.~\eqref{eq:app-sign-gate-ratio-statement}.
\end{proof}

%% file: sections/appendix_experiments.tex
\section{Numerical Analysis Details}
\label{app:numerical-analysis-details}
\label{app:experimental-details}

\input{sections/sign_gated_numerical_analysis}

\section{LLM Experiment Details}
\label{app:llm-experiment-details}

\subsection{Reasoning Self-Improvement Setup and Protocol}
\label{app:reasoning-experimental-details}

\paragraph{Scope and fairness controls.}
The large-scale experiments are designed to compare policy-update objectives rather than infrastructure choices. Within each backbone, all methods use the same training prompts, rollout group size, maximum response length, verifier, optimizer schedule, checkpoint cadence, and evaluation script. For methods that use privileged training information (OPSD, RLSD, and FlowBalance), that information is available only to a frozen scoring path during training. Rollout generation and evaluation never include the solution or feedback field $c$.

\begin{table}[t]
\centering
\small
\caption{Shared reporting fields for the reasoning experiments. Entries marked ``matched'' are held identical across methods within a backbone.}
\label{tab:reasoning-config-fields}
\setlength{\tabcolsep}{4pt}
\begin{tabular}{l p{0.62\linewidth}}
\toprule
Field & Setting \\
\midrule
Backbones & Qwen3-4B, Qwen3-8B \\
Training supervision & final-answer verifier; OPSD, RLSD, and FlowBalance also use privileged training solution/feedback \\
Compared methods & GRPO, OPSD, RLSD, FlowRL, FlowBalance \\
Rollout group size & matched across methods; pending release \\
Training horizon & step-180 final table; up to 400 steps for dynamics and response-length curves \\
Seeds & five seeds for every main-table entry \\
Maximum response length & matched across methods; pending release \\
Reference policy & fixed copy of the initial checkpoint \\
Training-time scorer & OPSD: fixed privileged teacher; RLSD/FlowBalance: frozen current-policy copy conditioned on $c$ \\
FlowBalance default coefficients & $\eta_A=15$, $\beta_G=1$ for the completed default run \\
Evaluation metrics & AIME24 Pass@16; HMMT25/Minerva/MATH500/OlympiadBench Pass@1 \\
Aggregate score & unweighted mean of the five benchmark means \\
\bottomrule
\end{tabular}
\end{table}

\paragraph{Training signal construction.}
For every prompt in a minibatch, the frozen rollout snapshot samples a group of responses without seeing $c$. A rule-based verifier assigns a terminal correctness reward, from which we compute the stopped group-relative advantage in Eq.~\eqref{eq:group-advantage}. FlowBalance then reuses the same frozen snapshot as the privileged-hindsight view by conditioning it on $c$ and scoring the already sampled tokens. The self-guidance contribution is the clipped, length-averaged log-probability gain in Eq.~\eqref{eq:trajectory-guidance-score}; no response is resampled from the guidance view and no guidance gradient is taken. The sign-gated trajectory energy is optimized through the trajectory-balance residual with the rollout-group partition estimate in Eq.~\eqref{eq:group-logz}.

\paragraph{Baseline implementations.}
GRPO uses the same verifier rewards and group-relative normalization as FlowBalance but optimizes the standard reward-policy objective rather than a normalized trajectory target. OPSD applies clipped forward-KL self-distillation from a fixed privileged teacher to the current policy's on-policy trajectories~\citep{zhao2026opsd}. RLSD combines verifier and teacher signals in its policy-optimization objective. FlowRL follows the original outcome-only trajectory-balance algorithm~\citep{zhu2025flowrl}. All baselines use the same sampled responses, answer verifier, and response-length cap whenever their objectives permit this sharing.

\paragraph{Evaluation protocol.}
For AIME24, each problem is decoded with 16 samples and is counted correct if any sample matches the final answer. For HMMT25, Minerva, MATH500, and OlympiadBench, we report Pass@1. The evaluator applies the same benchmark-specific final-answer extraction and normalization to every method. All main-table entries report the mean and sample standard deviation over five seeds at step 180. The accuracy-dynamics plots use matched evaluation checkpoints for FlowBalance and GRPO; the response-length panel compares the logged training trajectories of FlowBalance and OPSD. Solid curves denote smoothed trends and lighter curves denote the corresponding per-step measurements.

\subsection{Grounding and Guidance Ablations}
\label{app:additional-results}

\paragraph{Coefficient sweeps.}
The ablations in Tables~\ref{tab:ablation-eta-a}--\ref{tab:ablation-beta-q} are one-dimensional sweeps around the default setting. In the $\eta_A\in\{5,10,15\}$ sweep, $\beta_G$ and all other training settings are held fixed. In the $\beta_G\in\{1,2,3\}$ sweep, $\eta_A$ and all other settings are held fixed. Each sweep point is evaluated with the same five-benchmark average used in Table~\ref{tab:qwen-main}; all reported sweep points contain completed five-seed runs. Tables~\ref{tab:app-ablation-eta-a-full}--\ref{tab:app-ablation-beta-t-full} provide the corresponding per-benchmark breakdown.

\paragraph{Coefficient ablations.}
Tables~\ref{tab:app-ablation-eta-a-full}--\ref{tab:app-ablation-beta-t-full} expand the coefficient ablations in Section~\ref{sec:ablation-study} to the same five benchmark columns used in the main result table. All rows use the Qwen3-8B backbone and report step-180 evaluation. AIME24 is evaluated with Pass@16, while HMMT25, Minerva, MATH500, and OlympiadBench are evaluated with Pass@1. The average is the unweighted mean of the five benchmark means.

\begin{table*}[t]
\centering
\small
\setlength{\tabcolsep}{4pt}
\caption{Full ablation over verifier coefficient $\eta_A$ on Qwen3-8B at step 180. The self-guidance coefficient is held fixed at the default value $\beta_G=1$. Entries are percentages reported as mean $\pm$ sample standard deviation over five seeds. ``Avg.'' averages the five benchmark means.}
\label{tab:app-ablation-eta-a-full}
\begin{tabular}{l c cccc c}
\toprule
$\eta_A$ & AIME24@16 & HMMT25@1 & Minerva@1 & MATH500@1 & Olympiad@1 & Avg. \\
\midrule
5  & $86.00 \pm 1.49$
   & $30.00 \pm 8.50$ & $53.46 \pm 1.24$ & $93.32 \pm 0.41$ & $65.46 \pm 0.69$ & $65.65$ \\
10 & $85.33 \pm 1.83$
   & $30.00 \pm 3.33$ & $52.43 \pm 0.81$ & $93.20 \pm 0.32$ & $66.11 \pm 0.51$ & $65.41$ \\
15 & $89.33 \pm 1.49$
   & $34.67 \pm 9.89$ & $53.68 \pm 0.78$ & $93.52 \pm 0.30$ & $66.85 \pm 0.46$ & $67.61$ \\
\bottomrule
\end{tabular}
\end{table*}

\begin{table*}[t]
\centering
\small
\setlength{\tabcolsep}{4pt}
\caption{Full ablation over self-guidance coefficient $\beta_G$ on Qwen3-8B at step 180. The verifier coefficient is held fixed at the default value $\eta_A=15$. Entries are percentages reported as mean $\pm$ sample standard deviation over five seeds. ``Avg.'' averages the five benchmark means.}
\label{tab:app-ablation-beta-t-full}
\begin{tabular}{l c cccc c}
\toprule
$\beta_G$ & AIME24@16 & HMMT25@1 & Minerva@1 & MATH500@1 & Olympiad@1 & Avg. \\
\midrule
1 & $89.33 \pm 1.49$
  & $34.67 \pm 9.89$ & $53.68 \pm 0.78$ & $93.52 \pm 0.30$ & $66.85 \pm 0.46$ & $67.61$ \\
2 & $87.33 \pm 1.49$
  & $32.00 \pm 5.06$ & $53.68 \pm 1.40$ & $93.64 \pm 0.33$ & $65.73 \pm 1.16$ & $66.48$ \\
3 & $86.00 \pm 1.49$
  & $30.00 \pm 4.08$ & $53.46 \pm 1.39$ & $93.56 \pm 0.55$ & $66.74 \pm 0.66$ & $65.95$ \\
\bottomrule
\end{tabular}
\end{table*}

\subsection{Semantic Diversity Diagnostic}
\label{app:aime24-multisample-details}
The diversity study in Section~\ref{sec:diversity-study} evaluates semantic strategy variation on AIME24 using the step-180 checkpoints of GRPO, RLSD, and FlowBalance. We use seed $0$ only, decode $16$ complete responses for each of the $30$ problems and each method, and submit all $30\times3\times16=1440$ full trajectories to a GPT-5.5 judge. No response is truncated or heuristically compressed before judging. These results are therefore a controlled diagnostic of strategy variation at one checkpoint and one seed, not a multi-seed stability estimate.

\paragraph{Two-stage LLM-judge protocol.}
The judge first performs strategy extraction independently for each full trajectory. The extracted fields include the primary solution strategy, central mathematical representation, main tools or theorems, a short reasoning outline, a compact strategy signature, and, for incorrect responses, the attempted failure mode. The instruction is to describe the strategy actually used by the trajectory, not to repair incorrect reasoning. Differences in variable names, wording, notation, response length, arithmetic detail, or the order of equivalent calculations are not counted as different strategies.

In the second stage, for each problem and method, the $16$ extracted strategy summaries are anonymized and randomly ordered. The judge assigns every trajectory to exactly one semantic strategy cluster. During clustering, the judge does not see the source algorithm, checkpoint name, correctness label, or original response order. Correctness labels are applied only after clustering, and the reported diversity metric restricts the resulting cluster assignments to correct trajectories.

\paragraph{Metric.}
For a fixed problem and method, restrict the judged cluster assignments to correct trajectories and let $p_1,\ldots,p_K$ denote the proportions of those trajectories in the resulting semantic strategy clusters. We report correct-only Simpson strategy diversity,
\begin{equation}
    D_{\mathrm{Simpson}}
    =
    1-\sum_{k=1}^{K}p_k^2,
\end{equation}
which is the probability that two randomly sampled correct trajectories use different semantic solution strategies. A problem is included in the aggregate only when it has at least two correct trajectories.

\paragraph{Scope.}
LLM-based clustering is inherently partly subjective, even with anonymization and a two-stage protocol. The main-text conclusion is therefore restricted to this AIME24 seed-$0$ diagnostic: FlowBalance exhibits higher judged semantic strategy diversity among correct trajectories at this checkpoint.

\subsection{Prompt for Semantic Strategy Clustering}
\label{app:diversity-evaluation-prompt}

The LLM-judged diagnostic uses two prompts: one extracts a strategy summary from each trajectory, and the other clusters anonymized summaries within the same problem and method.

\begin{tcolorbox}[
    colback=gray!3,
    colframe=black!55,
    boxrule=0.5pt,
    arc=1.5pt,
    left=7pt,
    right=7pt,
    top=6pt,
    bottom=6pt,
    title={\textbf{Prompt 1: Per-trajectory strategy extraction}},
    fonttitle=\small,
    coltitle=black,
    colbacktitle=gray!12,
    before skip=0.5em,
    after skip=0.8em
]
\footnotesize
Given the problem statement and one complete reasoning trajectory, describe the mathematical strategy actually attempted by the trajectory. Do not repair incorrect reasoning, replace the attempted method with a cleaner solution, or infer a method not present in the text.

Extract: (i) the primary solution strategy, (ii) the central mathematical representation, (iii) the main tools or theorems, (iv) a short reasoning outline, (v) a compact strategy signature, and (vi) for incorrect responses, the attempted failure mode.

Treat changes in notation, variable names, wording, response length, arithmetic detail, or the order of equivalent calculations as the same strategy. Treat substantively different mathematical objects or tools---for example, rectangular-box embedding versus Cayley--Menger determinant, coordinate versus synthetic geometry, generating functions versus direct counting, or roots-of-unity methods versus real polynomial methods---as different strategies.
\end{tcolorbox}

\begin{tcolorbox}[
    colback=gray!3,
    colframe=black!55,
    boxrule=0.5pt,
    arc=1.5pt,
    left=7pt,
    right=7pt,
    top=6pt,
    bottom=6pt,
    title={\textbf{Prompt 2: Within-problem strategy clustering}},
    fonttitle=\small,
    coltitle=black,
    colbacktitle=gray!12,
    before skip=0.2em,
    after skip=0.5em
]
\footnotesize
Given 16 anonymized strategy summaries for the same problem and method, assign every trajectory to exactly one semantic strategy cluster. The summaries are randomly ordered and do not reveal the source algorithm, checkpoint name, correctness label, or original response order.

Cluster by the core mathematical representation and main tools, not by surface wording, formatting, or response length. Incorrect but coherent attempts should still be clustered by their attempted strategy; correctness labels are applied only after clustering when computing the correct-only diversity metric.

Return the cluster assignment for each anonymous trajectory, a short cluster name, and a brief rationale for each cluster.
\end{tcolorbox}

\subsection{Why the FlowBalance Traces Are Interesting}
\label{app:additional-diversity-cases}

The main text shows AIME24 Problem 23 as a compact illustration of the semantic distinctions used by the LLM-judged diversity diagnostic. Here we provide three additional FlowBalance cases from the same seed-$0$, step-$180$, $16$-sample AIME24 run. Each case contains two representative correct trajectories whose final answers agree but whose central mathematical objects differ; the boxes mark the strategy-defining reasoning actions, following the format of Table~\ref{tab:aime24-case-study}.

\newcommand{\caseboxgray}[1]{\fcolorbox{black!35}{white}{\begin{minipage}[t]{\dimexpr\linewidth-2\fboxsep-2\fboxrule\relax}\strut #1\strut\end{minipage}}}
\newcommand{\caseboxblue}[1]{\fcolorbox{blue!45!black}{white}{\begin{minipage}[t]{\dimexpr\linewidth-2\fboxsep-2\fboxrule\relax}\strut #1\strut\end{minipage}}}
\newcommand{\casestep}[1]{\caseboxgray{#1}\\[1.5pt]}
\newcommand{\casebstep}[1]{\caseboxblue{#1}\\[1.5pt]}

Together with the AIME24 Problem 23 example in the main text, these cases show that FlowBalance's additional correct trajectories are not merely longer, more verbose, or differently worded variants of a standard template. We regard a trace as \emph{interesting} when the boxed steps change the mathematical object or criterion that drives the solution: for example, replacing a local algebraic test with a global envelope argument, replacing a three-dimensional geometry problem with a two-dimensional cross-section reduction, or replacing direct elimination with a parameterized conic argument. In other words, the diversity signal is intended to capture a semantic change in \emph{how} the problem is solved, not a superficial change in phrasing.

On Problem 5, FlowBalance finds both a global envelope view of the unit-intercept line family and a local multiplicity view. The first identifies the special point as the tangency point on the astroid $x^{2/3}+y^{2/3}=1$, while the second detects uniqueness by forcing the known line $AB$ to be a double root of a one-variable trigonometric equation. On Problem 15, one FlowBalance trajectory reduces the three-dimensional torus-sphere contact to two circle tangencies in a meridian plane. Another keeps the problem as an implicit surface calculation and derives the contact radii from collinear normals in cylindrical coordinates. On Problem 27, FlowBalance again covers two different representations of the same constraint: direct coordinate elimination in squared variables, and a secant-tangent hyperbola parameterization that turns the rhombus condition into the trigonometric product $\sin\theta\sin\phi=-5/6$. These examples mirror the main-text box-embedding case and help explain why the FlowBalance traces are meaningful: the model is spreading probability mass across genuinely different correct derivations, not just stylistic rewrites of one template.

\paragraph{How to read the traces.} The boxed lines mark only the reasoning actions that define the route; standard algebra, routine simplifications, and repeated bookkeeping are suppressed with ``$\cdots$''. The goal is not to claim that every token sequence differs globally, but to show that the high-level derivational backbone differs in a way that would matter to a human reader choosing between strategies.

\begin{table}[h]
\centering
\scriptsize
\setlength{\tabcolsep}{5pt}
\renewcommand{\arraystretch}{1.18}
\caption{\textbf{Additional case study on AIME24 Problem 5.} Two correct FlowBalance trajectories solve the same unit-intercept-segment problem using different criteria for uniqueness.}
\label{tab:aime24-case-05}
\begin{tabularx}{\textwidth}{p{0.12\textwidth}X}
\toprule
& \centering\textbf{Content (boxed = reasoning actions; ``$\cdots$'' = omitted)}\tabularnewline
\midrule
\textbf{Question}
&
Let $O=(0,0)$, $A=(1/2,0)$, and $B=(0,\sqrt3/2)$. Among all unit-length segments with one endpoint on the positive $x$-axis and one endpoint on the positive $y$-axis, there is a unique interior point $C$ of $AB$ that lies on no other such segment. If $OC^2=p/q$ in lowest terms, find $p+q$.
\tabularnewline
\midrule
\rowcolor{black!5}
\textbf{Route 1}
&
\begin{minipage}[t]{\linewidth}
\textbf{Envelope / astroid route.} ``$\cdots$''\\[1.5pt]
\casestep{$F(x,y,u)=x/u+y/\sqrt{1-u^2}-1$ for the unit-intercept line family.}
\casestep{Solve $F=0$ and $F_u=0$ to obtain the envelope.}
\casestep{$x=u^3$, $y=(1-u^2)^{3/2}$, the first-quadrant astroid.}
\casestep{$AB$ corresponds to $u=1/2$, so $C=(1/8,3\sqrt3/8)$.}
\caseboxgray{$OC^2=7/16$, hence $p+q=23$.}
\end{minipage}
\tabularnewline
\midrule
\rowcolor{blue!10}
\textbf{Route 2}
&
\begin{minipage}[t]{\linewidth}
\textbf{Multiple-root uniqueness route.} ``$\cdots$''\\[1.5pt]
\casebstep{Write an interior point as $C(t)=((1-t)/2,\sqrt3t/2)$.}
\casebstep{A unit-intercept line through $C(t)$ satisfies $(1-t)/(2\cos\theta)+\sqrt3t/(2\sin\theta)=1$.}
\casebstep{The segment $AB$ is the known solution $\theta=\pi/3$.}
\casebstep{Uniqueness at the boundary forces this solution to be a double root.}
\casebstep{$f'(\pi/3)=\sqrt3(1-t)-\sqrt3t/3=0$, so $t=3/4$.}
\caseboxblue{$C=(1/8,3\sqrt3/8)$ and $p+q=23$.}
\end{minipage}
\tabularnewline
\bottomrule
\end{tabularx}
\end{table}

\begin{table}[h]
\centering
\scriptsize
\setlength{\tabcolsep}{5pt}
\renewcommand{\arraystretch}{1.18}
\caption{\textbf{Additional case study on AIME24 Problem 15.} Two correct FlowBalance trajectories use either a two-dimensional meridian section or a three-dimensional implicit-normal condition.}
\label{tab:aime24-case-15}
\begin{tabularx}{\textwidth}{p{0.12\textwidth}X}
\toprule
& \centering\textbf{Content (boxed = reasoning actions; ``$\cdots$'' = omitted)}\tabularnewline
\midrule
\textbf{Question}
&
A torus is formed by rotating a circle of radius $3$ about an axis in its plane at distance $6$ from the circle's center. A sphere of radius $11$ is tangent to the torus in two rotationally symmetric positions, with corresponding contact-circle radii $r_i$ and $r_o$. If $r_i-r_o=m/n$ in lowest terms, find $m+n$.
\tabularnewline
\midrule
\rowcolor{black!5}
\textbf{Route 1}
&
\begin{minipage}[t]{\linewidth}
\textbf{Meridian-section circle-tangency route.} ``$\cdots$''\\[1.5pt]
\casestep{Take a plane through the rotation axis; the torus cross-section is a circle of radius $3$ centered at $(6,0)$.}
\casestep{The sphere cross-section is a circle of radius $11$ centered at $(0,h)$.}
\casestep{The two tangencies have center distance $11-3=8$ or $11+3=14$.}
\casestep{The contact-circle radius is the horizontal coordinate of the tangency point.}
\casestep{$r_i=11\cdot6/8=33/4$ and $r_o=11\cdot6/14=33/7$.}
\caseboxgray{$r_i-r_o=99/28$, so $m+n=127$.}
\end{minipage}
\tabularnewline
\midrule
\rowcolor{blue!10}
\textbf{Route 2}
&
\begin{minipage}[t]{\linewidth}
\textbf{Implicit-surface normal route.} ``$\cdots$''\\[1.5pt]
\casebstep{Use cylindrical coordinates $(\rho,z)$.}
\casebstep{Torus: $(\rho-6)^2+z^2=9$; sphere: $\rho^2+(z-h)^2=121$.}
\casebstep{Tangency means normals are collinear: $(\rho-6,z)=\lambda(\rho,z-h)$.}
\casebstep{Eliminate $z,h$ to get $\rho/|\rho-6|=11/3$.}
\casebstep{$\rho>6$ gives $\rho=33/4$; $\rho<6$ gives $\rho=33/7$.}
\caseboxblue{$m+n=127$.}
\end{minipage}
\tabularnewline
\bottomrule
\end{tabularx}
\end{table}

\begin{table}[h]
\centering
\scriptsize
\setlength{\tabcolsep}{5pt}
\renewcommand{\arraystretch}{1.18}
\caption{\textbf{Additional case study on AIME24 Problem 27.} Two correct FlowBalance trajectories optimize the same diagonal length using algebraic elimination or a trigonometric hyperbola parameterization.}
\label{tab:aime24-case-27}
\begin{tabularx}{\textwidth}{p{0.12\textwidth}X}
\toprule
& \centering\textbf{Content (boxed = reasoning actions; ``$\cdots$'' = omitted)}\tabularnewline
\midrule
\textbf{Question}
&
Points $A,B,C,D$ lie on the hyperbola $x^2/20-y^2/24=1$ and form a rhombus whose diagonals intersect at the origin. Find the greatest real number that is strictly less than $BD^2$ for every such rhombus.
\tabularnewline
\midrule
\rowcolor{black!5}
\textbf{Route 1}
&
\begin{minipage}[t]{\linewidth}
\textbf{Coordinate-elimination route.} ``$\cdots$''\\[1.5pt]
\casestep{Write opposite vertices as $(a,b),(-a,-b)$ and $(c,d),(-c,-d)$.}
\casestep{A centered parallelogram is a rhombus iff its diagonals are perpendicular, so $ac+bd=0$.}
\casestep{Set $x=a^2$ and use $b^2=6x/5-24$.}
\casestep{From perpendicularity, $c^2=(b^2/a^2)d^2$.}
\casestep{The hyperbola equation for $B$ gives $d^2=600x/(11x-720)$.}
\caseboxgray{$BD^2=4(c^2+d^2)=480+288000/(11x-720)>480$, with limit $480$ from above.}
\end{minipage}
\tabularnewline
\midrule
\rowcolor{blue!10}
\textbf{Route 2}
&
\begin{minipage}[t]{\linewidth}
\textbf{$\sec$--$\tan$ parameterization route.} ``$\cdots$''\\[1.5pt]
\casebstep{Parameterize the hyperbola as $(2\sqrt5\sec\theta,2\sqrt6\tan\theta)$.}
\casebstep{Take $A$ with parameter $\theta$ and $B$ with parameter $\phi$.}
\casebstep{Perpendicularity gives $20\sec\theta\sec\phi+24\tan\theta\tan\phi=0$.}
\casebstep{Equivalently, $\sin\theta\sin\phi=-5/6$.}
\casebstep{Since $|\sin\theta|<1$, we have $|\sin\phi|>5/6$, hence $\tan^2\phi>25/11$.}
\caseboxblue{$BD^2=80+176\tan^2\phi>480$, with limit $480$ attainable only at infinity.}
\end{minipage}
\tabularnewline
\bottomrule
\end{tabularx}
\end{table}

%% file: sections/sign_gated_numerical_analysis.tex
\subsection{Exact Numerical Diagnostics for Verifier-Grounded Self-Improvement}
\label{sec:sign-gated-numerics}

We instantiate the unchanged FlowBalance target on exactly enumerable response spaces. Every distribution below is obtained by normalizing
\begin{equation}
    q(y)
    \propto
    \rho(y)
    \exp\!\left(
      \frac{\eta A(y)+\beta G(y)\sign(A(y))}{\tau}
    \right),
    \label{eq:numerical-exact-target}
\end{equation}
with no surrogate objective or alternative gate. These are qualitative diagnostics of the target, not additional language-model training runs. We organize them as a short narrative. The first plot gives the basic geometric intuition for verifier-grounded self-guidance. The second isolates the exact anti-self-confirmation effect in the simplest mixed-outcome group. The third maps the reliability boundary of self-guidance. The fourth returns to the two structural properties emphasized by the theory: conservative reference change and efficient use of all within-group contrasts.

\subsubsection{Mechanism View: Grounded Self-Guidance across Multiple Modes}

Figure~\ref{fig:flowbalance-gaussian-geometry} embeds four exact response modes as one-dimensional Gaussians: two failures and two successes. The verifier cannot distinguish the two successful modes, while privileged hindsight assigns its largest gain to the more robust success. It also assigns positive local likelihood to plausible failures, which is exactly the false-positive regime addressed by Proposition~\ref{prop:sign-gate-pairwise}. We use
\begin{equation}
    \rho=(0.30,0.20,0.30,0.20),\quad
    A=(-1,-1,+1,+1),\quad
    G=(0.45,0.25,0.30,0.75),
\end{equation}
with $(\eta,\beta,\tau)=(0.75,0.80,1)$. Reward-only shaping raises total success mass to $0.818$ but cannot prefer the robust successful mode. Ungated self-guidance identifies that mode, yet also reinforces positive guidance on failures, reaching success mass $0.832$. FlowBalance preserves the useful ranking among successes while reversing the failed-response contribution, reaching success mass $0.900$ and robust-success mass $0.440$. The Gaussian curves visualize exact categorical masses; they do not approximate the target.

\begin{figure}[H]
    \centering
    \includegraphics[width=0.96\linewidth]{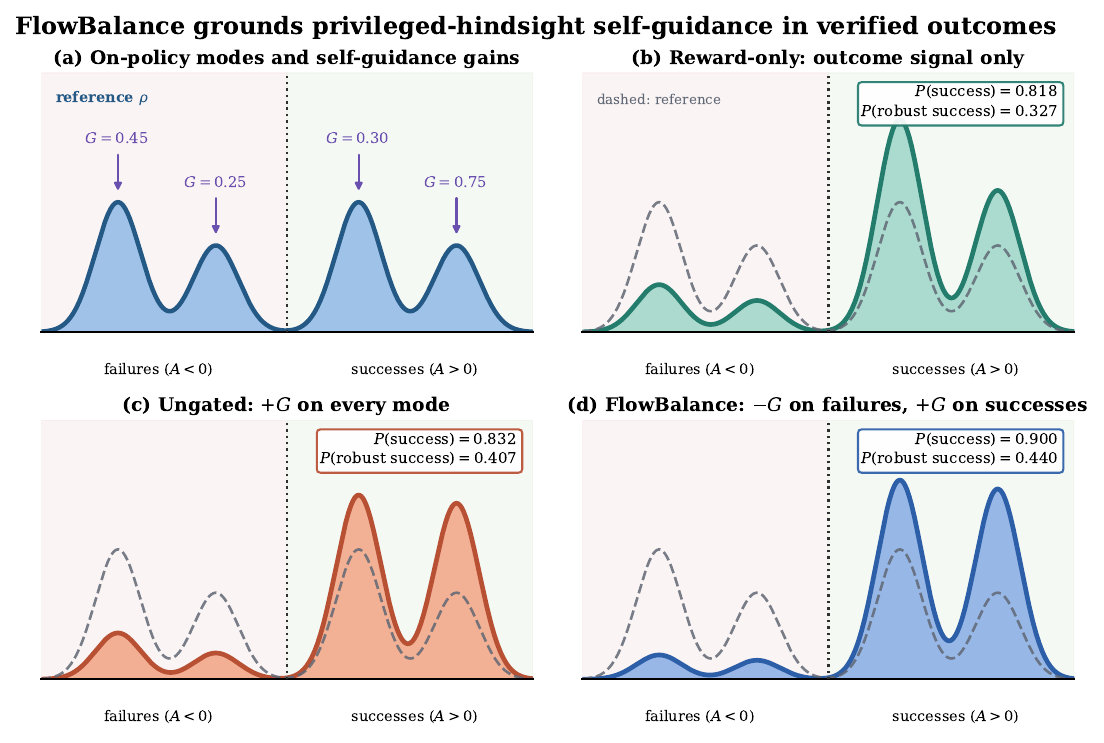}
    \caption{\textbf{Mechanism view of verifier-grounded self-guidance.} FlowBalance uses $+G$ on positive-advantage responses and $-G$ on negative-advantage responses. It retains useful guidance among successful modes without converting positive local confidence on failed responses into self-reinforcement.}
    \label{fig:flowbalance-gaussian-geometry}
\end{figure}

\subsubsection{Pairwise View: Exact Anti-Self-Confirmation in a Mixed Group}

Figure~\ref{fig:flowbalance-sign-gate-phase} fixes positive self-guidance $G_+>0$ on a successful response and sweeps positive false-support $G_->0$ on a plausible failure. Proposition~\ref{prop:sign-gate-pairwise} gives the exact log-odds advantages
\begin{equation}
    \logit p_{\mathrm{FlowBalance}}-\logit p_{\mathrm{reward}}
    =\frac{\beta(G_++G_-)}{\tau}>0,
    \qquad
    \logit p_{\mathrm{FlowBalance}}-\logit p_{\mathrm{ungated}}
    =\frac{2\beta G_-}{\tau}>0.
    \label{eq:numerical-dominance-margins}
\end{equation}
Thus FlowBalance dominates both baselines throughout this useful-but-imperfect quadrant. As false-positive support grows, ungated shaping becomes less reward-aligned, whereas outcome calibration becomes more selective. At $G_-=0.5$, the exact target success probabilities are $0.894$ for FlowBalance, $0.817$ for reward-only shaping, and $0.807$ for ungated shaping.

\begin{figure}[H]
    \centering
    \includegraphics[width=0.96\linewidth]{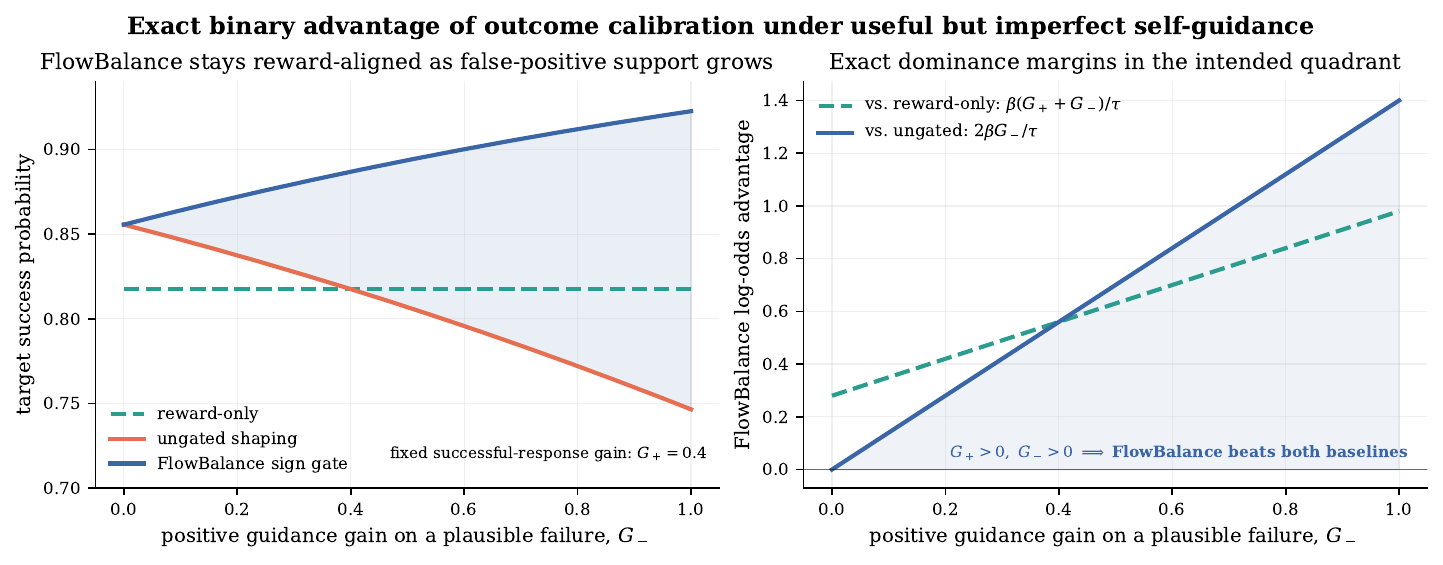}
    \caption{\textbf{Exact anti-self-confirmation advantage.} The left panel shows target success probability as positive self-guidance on a failed response increases. The right panel plots the closed-form probability-ratio margins in Eq.~\eqref{eq:numerical-dominance-margins}; no fitted decision boundary is used.}
    \label{fig:flowbalance-sign-gate-phase}
\end{figure}

\subsubsection{Boundary View: When Self-Guidance Helps and When It Stops Helping}

Outcome calibration does not make arbitrary guidance safe. Figure~\ref{fig:flowbalance-self-guidance-reliability} makes this boundary explicit on a five-mode response space. We interpolate the guidance vector as $G_q=qG_{\mathrm{useful}}+(1-q)G_{\mathrm{adv}}$, where $q=1$ is useful but still false-positive on failures and $q=0$ is adversarial to the verifier. We then sweep both $q$ and $\beta_G/\tau$. Above a reliability threshold, FlowBalance improves verified-success mass over reward-only shaping while retaining several successful modes. Below that threshold, stronger guidance can reduce verified success. The reverse-KL panel quantifies the accompanying movement from the reference. This diagnostic therefore supports a bounded claim: sign gating corrects false-positive confidence on failures, but it does not justify arbitrarily strong guidance that is systematically wrong on successful responses.

\begin{figure}[H]
    \centering
    \includegraphics[width=0.98\linewidth]{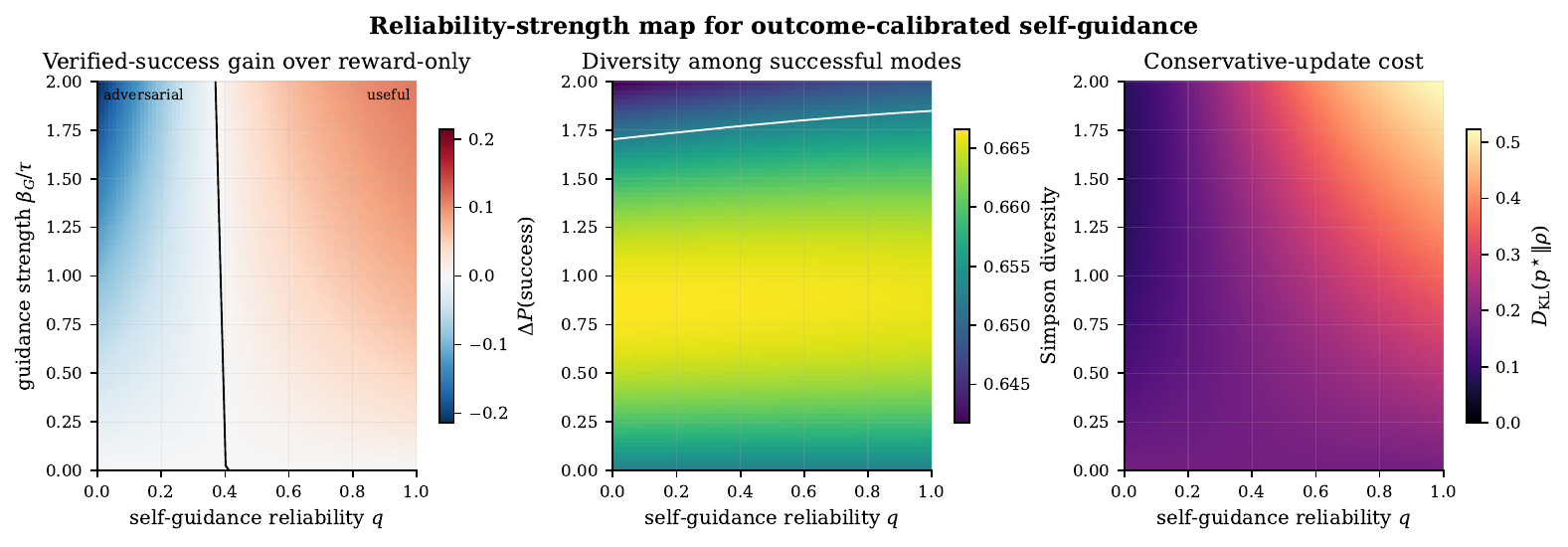}
    \caption{\textbf{Reliability--strength map for self-guidance.} Left: gain in verified-success mass over reward-only shaping. Middle: conditional Simpson diversity among successful modes; the white contour marks the reward-only level. Right: reverse KL to the reference. The reliability parameter is a synthetic interpolation, not an empirical calibration estimate.}
    \label{fig:flowbalance-self-guidance-reliability}
\end{figure}

\subsubsection{Structural View: Conservative Change and All-Contrast Efficiency}

Figure~\ref{fig:flowbalance-compact-theory-diagnostics} verifies two distributional consequences of the FlowBalance formulation. First, the exponential-tilt path traces the minimum reverse-KL frontier as expected sign-gated energy increases. A constructed full-support fixed-data target is matched to the exact FlowBalance energy, but requires reverse KL $0.973$ instead of $0.273$, a $3.6\times$ larger displacement from the reference. Second, a local Gaussian calculation shows that profiling the group partition intercept retains all $N-1$ within-group contrast directions. At group size $N=32$, its exact parameter risk is $2.92\%$ of a one-contrast-per-group estimator, approximately $34\times$ lower risk. Together, these diagnostics explain why the self-improvement update is both conservative in where it moves probability mass and statistically efficient in how it uses each rollout group.

\begin{figure}[H]
    \centering
    \includegraphics[width=0.96\linewidth]{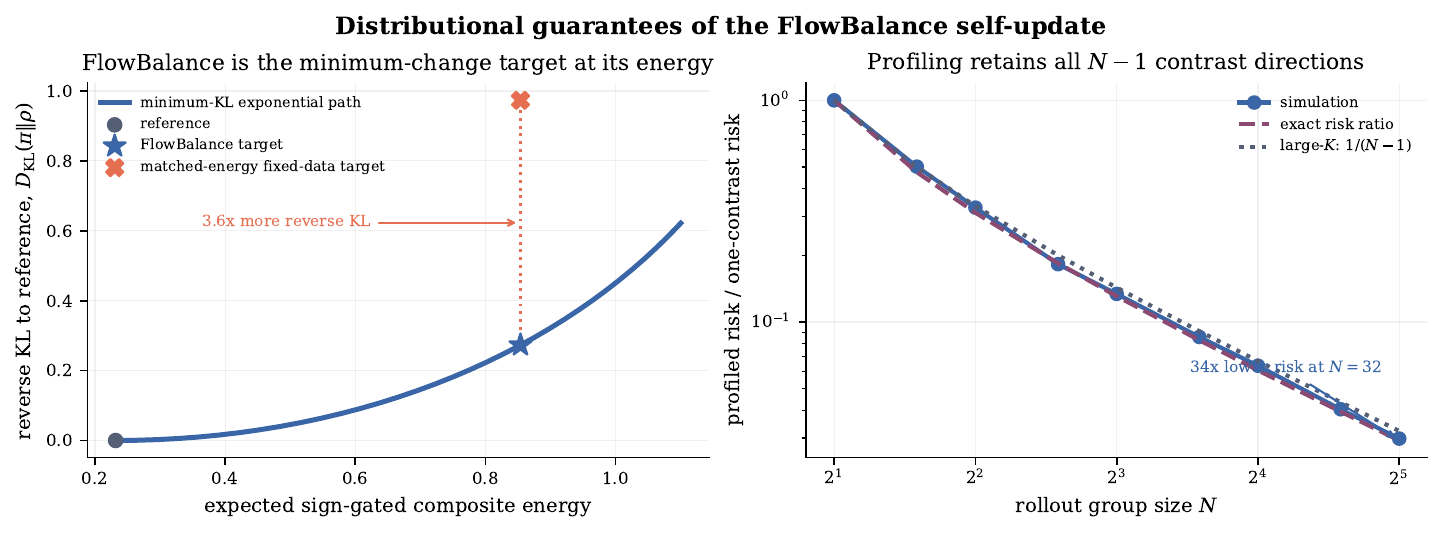}
    \caption{\textbf{Conservative change and all-contrast fitting.} Left: FlowBalance is the minimum reverse-KL target at its attained composite-energy level. Right: profiling one nuisance intercept per rollout group preserves all $N-1$ contrast directions and yields the exact local Gaussian risk reduction.}
    \label{fig:flowbalance-compact-theory-diagnostics}
\end{figure}

%% file: paper.bib
@STRING{TIP = "IEEE Trans. Image Process." }

@STRING{ICML  = "Int. Conf. Mach. Learn."}

@string(AAAI  = "{AAAI Conf. Art. Intell.}")

@misc{zhao2026opsd,
  title        = {Self-Distilled Reasoner: On-Policy Self-Distillation for Large Language Models},
  author       = {Zhao, Siyan and Xie, Zhihui and Liu, Mengchen and Huang, Jing and Pang, Guan and Chen, Feiyu and Grover, Aditya},
  year         = {2026},
  eprint       = {2601.18734},
  archivePrefix= {arXiv},
  primaryClass = {cs.CL}
}

@misc{hubotter2026sdpo,
  title        = {Reinforcement Learning via Self-Distillation},
  author       = {H{\"u}botter, Jonas and L{\"u}beck, Frederike and Behric, Lejs and Baumann, Anton and Bagatella, Marco and Marta, Daniel and Hakimi, Ido and Shenfeld, Idan and Kleine Buening, Thomas and Guestrin, Carlos and Krause, Andreas},
  year         = {2026},
  eprint       = {2601.20802},
  archivePrefix= {arXiv},
  primaryClass = {cs.LG}
}

@misc{shenfeld2026sdft,
  title        = {Self-Distillation Enables Continual Learning},
  author       = {Shenfeld, Idan and Damani, Mehul and H{\"u}botter, Jonas and Agrawal, Pulkit},
  year         = {2026},
  eprint       = {2601.19897},
  archivePrefix= {arXiv},
  primaryClass = {cs.LG}
}

@misc{kim2026degradation,
  title        = {Why Does Self-Distillation (Sometimes) Degrade the Reasoning Capability of {LLM}s?},
  author       = {Kim, Jeonghye and Luo, Xufang and Kim, Minbeom and Lee, Sangmook and Kim, Dohyung and Jeon, Jiwon and Li, Dongsheng and Yang, Yuqing},
  year         = {2026},
  eprint       = {2603.24472},
  archivePrefix= {arXiv},
  primaryClass = {cs.CL}
}

@misc{yu2026dopd,
  title        = {{DOPD}: Dual On-policy Distillation},
  author       = {Yu, Xinlei and Li, Gen and Si, Qingyi and Zhang, Guibin and Xu, Yuqi and Wang, Congcong and Dong, Shuai and Tuo, Kaiwen and Zeng, Xiangyu and Feng, Kaituo and Wang, Qunzhong and Shi, Yang and Hu, Xiaobin and Yue, Xiangyu and Wang, Jiaqi and Yan, Shuicheng},
  year         = {2026},
  eprint       = {2606.30626},
  archivePrefix= {arXiv},
  primaryClass = {cs.CL}
}

@misc{zhu2026manyfaces,
  title        = {The Many Faces of On-Policy Distillation: Pitfalls, Mechanisms, and Fixes},
  author       = {Zhu, Siqi and Ye, Xuyan and Lu, Hongyu and Shi, Weiye and Liu, Ge},
  year         = {2026},
  eprint       = {2605.11182},
  archivePrefix= {arXiv},
  primaryClass = {cs.CL}
}

@misc{zhao2026rosd,
  title        = {Reflective On-Policy Self-Distillation for Language Model Reasoning across Domains},
  author       = {Zhao, Ziqi and Ma, Xinyu and Yang, Liu and Feng, Yujie and Shi, Daiting and He, Jingzhou and Xin, Xin and Ren, Zhaochun and Wu, Xiao-Ming},
  year         = {2026},
  eprint       = {2605.28014},
  archivePrefix= {arXiv},
  primaryClass = {cs.CL}
}

@misc{zhang2026dasd,
  title        = {Tailoring Teaching to Aptitude: Direction-Adaptive Self-Distillation for LLM Reasoning},
  author       = {Zhang, Hongbin and Wang, Chaozheng and Chen, Kehai and Pan, Youcheng and Xiang, Yang and Wang, Jinpeng and Zhang, Min},
  year         = {2026},
  eprint       = {2605.22263},
  archivePrefix= {arXiv},
  primaryClass = {cs.CL}
}

@misc{kim2026compresses,
  title        = {{OPSD} Compresses What {RLVR} Teaches: A Post-{RL} Compaction Stage for Reasoning Models},
  author       = {Kim, Jaehoon and Lee, Dongha},
  year         = {2026},
  eprint       = {2605.06188},
  archivePrefix= {arXiv},
  primaryClass = {cs.CL}
}

@misc{jia2026aopd,
  title        = {Asymmetric On-Policy Distillation: Bridging Exploitation and Imitation at the Token Level},
  author       = {Jia, Nan and Yang, Haojin and Ma, Xing and Lian, Jiesong and Zhang, Shuailiang and Zhang, Weipeng and Zeng, Ke and Cai, Xunliang and Sun, Zequn},
  year         = {2026},
  eprint       = {2605.06387},
  archivePrefix= {arXiv},
  primaryClass = {cs.CL}
}

@misc{zhu2025flowrl,
  title        = {{FlowRL}: Matching Reward Distributions for {LLM} Reasoning},
  author       = {Zhu, Xuekai and Cheng, Daixuan and Zhang, Dinghuai and Li, Hengli and Zhang, Kaiyan and Jiang, Che and Sun, Youbang and Hua, Ermo and Zuo, Yuxin and Lv, Xingtai and Zhang, Qizheng and Chen, Lin and Shao, Fanghao and Xue, Bo and Song, Yunchong and Yang, Zhenjie and Cui, Ganqu and Ding, Ning and Gao, Jianfeng and Liu, Xiaodong and Zhou, Bowen and Mei, Hongyuan and Lin, Zhouhan},
  year         = {2025},
  eprint       = {2509.15207},
  archivePrefix= {arXiv},
  primaryClass = {cs.LG}
}

@misc{yu2025dapo,
  title        = {{DAPO}: An Open-Source {LLM} Reinforcement Learning System at Scale},
  author       = {Yu, Qiying and Zhang, Zheng and Zhu, Ruofei and Yuan, Yufeng and Zuo, Xiaochen and Yue, Yu and Fan, Tiantian and Liu, Gaohong and Liu, Lingjun and Liu, Xin and Lin, Haibin and Lin, Zhiqi and Ma, Bole and Sheng, Guangming and Tong, Yuxuan and Zhang, Chi and Zhang, Mofan and Zhang, Wang and Zhu, Hang and Zhu, Jinhua and Chen, Jiaze and Chen, Jiangjie and Wang, Chengyi and Li, Hongli and Dai, Weinan and Song, Yuxuan and Wei, Xiangpeng and Zhou, Hao and Liu, Jingjing and Ma, Wei-Ying and Zhang, Ya-Qin and Lin, Yan and Qiao, Mu and Wu, Yonghui and Wang, Mingxuan},
  year         = {2025},
  eprint       = {2503.14476},
  archivePrefix= {arXiv},
  primaryClass = {cs.CL}
}

@misc{shao2024deepseekmath,
  title        = {{DeepSeekMath}: Pushing the Limits of Mathematical Reasoning in Open Language Models},
  author       = {Shao, Zhihong and Wang, Peiyi and Zhu, Qihao and Xu, Runxin and Song, Junxiao and Bi, Xiao and Zhang, Haowei and Zhang, Mingchuan and Li, Y. K. and Wu, Y. and Guo, Daya},
  year         = {2024},
  eprint       = {2402.03300},
  archivePrefix= {arXiv},
  primaryClass = {cs.CL}
}

@misc{kim2026rebellious,
  title        = {Rebellious Student: Reversing Teacher Signals for Reasoning Exploration with Self-Distilled {RLVR}},
  author       = {Kim, Jeonghye and Jeon, Jiwon and Li, Dongsheng and Yang, Yuqing},
  year         = {2026},
  eprint       = {2605.10781},
  archivePrefix= {arXiv},
  primaryClass = {cs.CL}
}

@misc{xu2026tip,
  title        = {{TIP}: Token Importance in On-Policy Distillation},
  author       = {Xu, Yuanda and Sang, Hejian and Zhou, Zhengze and He, Ran and Wang, Zhipeng and Geramifard, Alborz},
  year         = {2026},
  eprint       = {2604.14084},
  archivePrefix= {arXiv},
  primaryClass = {cs.LG}
}

@misc{yang2025qwen3,
  title        = {{Qwen3} Technical Report},
  author       = {Yang, An and Li, Anfeng and Yang, Baosong and Zhang, Beichen and Hui, Binyuan and Zheng, Bo and Yu, Bowen and Gao, Chang and Huang, Chengen and Lv, Chenxu and Zheng, Chujie and Liu, Dayiheng and Zhou, Fan and Huang, Fei and Hu, Feng and Ge, Hao and Wei, Haoran and Lin, Huan and Tang, Jialong and Yang, Jian and Tu, Jianhong and Zhang, Jianwei and Yang, Kexin and Yu, Le and Wang, Xinyu and Zhou, Jing and Qwen Team},
  year         = {2025},
  eprint       = {2505.09388},
  archivePrefix= {arXiv},
  primaryClass = {cs.CL}
}

@misc{yang2026rlsd,
  title         = {Self-Distilled {RLVR}},
  author        = {Yang, Chenxu and Qin, Chuanyu and Si, Qingyi and Chen, Minghui and Gu, Naibin and Yao, Dingyu and Lin, Zheng and Wang, Weiping and Wang, Jiaqi and Duan, Nan},
  year          = {2026},
  eprint        = {2604.03128},
  archivePrefix = {arXiv},
  primaryClass  = {cs.LG}
}

@misc{pan2026rlcsd,
  title         = {{RLCSD}: Reinforcement Learning with Contrastive On-Policy Self-Distillation},
  author        = {Pan, Leyi and Tao, Shuchang and Zhai, Yunpeng and Zhang, Shun and Li, Wenzhe and Liu, Xiao and Wang, Jibin and Shen, Jian and Tu, Zhaopeng and Zhao, Kaisheng},
  year          = {2026},
  eprint        = {2606.11709},
  archivePrefix = {arXiv},
  primaryClass  = {cs.CL}
}

@misc{du2026gflowrl,
  title         = {{GFlowRL}: Scaling Distribution-Matching {RL} to Large Language Models},
  author        = {Du, Li and Li, Kanglong and Cao, Yihang and Lin, Kaiyan and Ma, Yuanqi and Lu, Yuxuan and Chen, Xi and Yang, Jinyang and Ma, Zhiyuan and Chu, Zhixuan and Huang, Xuanjing and Yu, Hui and Zhang, Jie M. and Wei, Furu and Yu, Bowen and Li, Lei},
  year          = {2026},
  eprint        = {2607.13394},
  archivePrefix = {arXiv},
  primaryClass  = {cs.CL}
}

@misc{zheng2025gspo,
  title         = {Group Sequence Policy Optimization},
  author        = {Zheng, Chujie and Liu, Shixuan and Li, Mingze and Chen, Xiong-Hui and Yu, Bowen and Gao, Chang and Dang, Kai and Liu, Yuqiong and Men, Rui and Yang, An and Zhou, Jingren and Lin, Junyang},
  year          = {2025},
  eprint        = {2507.18071},
  archivePrefix = {arXiv},
  primaryClass  = {cs.LG}
}

@misc{xie2025spo,
  title         = {Simple Policy Optimization},
  author        = {Xie, Zhengpeng and Zhang, Qiang and Yang, Fan and Hutter, Marco and Xu, Renjing},
  year          = {2025},
  eprint        = {2401.16025},
  archivePrefix = {arXiv},
  primaryClass  = {cs.LG}
}

@misc{zhang2025gvpo,
  title         = {{GVPO}: Group Variance Policy Optimization for Large Language Model Post-Training},
  author        = {Zhang, Kaichen and Hong, Yuzhong and Bao, Junwei and Jiang, Hongfei and Song, Yang and Hong, Dingqian and Xiong, Hui},
  year          = {2025},
  eprint        = {2504.19599},
  archivePrefix = {arXiv},
  primaryClass  = {cs.AI}
}

@misc{yang2025dcpo,
  title         = {{DCPO}: Dynamic Clipping Policy Optimization},
  author        = {Yang, Shihui and Dou, Chengfeng and Guo, Peidong and Lu, Kai and Ju, Qiang and Deng, Fei and Xin, Rihui},
  year          = {2025},
  eprint        = {2509.02333},
  archivePrefix = {arXiv},
  primaryClass  = {cs.CL}
}

@misc{huang2026oblr,
  title         = {Variance-Aware Baselines and Adaptive Learning Rates for Reinforcement Learning with Verifiable Rewards},
  author        = {Huang, Zixun and Sheng, Jiayi and Zheng, Zeyu},
  year          = {2026},
  eprint        = {2511.23310},
  archivePrefix = {arXiv},
  primaryClass  = {stat.ML}
}

@inproceedings{bengio2021gflownet,
  title     = {Flow Network Based Generative Models for Non-Iterative Diverse Candidate Generation},
  author    = {Bengio, Emmanuel and Jain, Moksh and Korablyov, Maksym and Precup, Doina and Bengio, Yoshua},
  booktitle = {Advances in Neural Information Processing Systems},
  volume    = {34},
  pages     = {27381--27394},
  year      = {2021}
}

@article{bengio2023gflownet,
  title   = {{GFlowNet} Foundations},
  author  = {Bengio, Yoshua and Lahlou, Salem and Deleu, Tristan and Hu, Edward J. and Tiwari, Mo and Bengio, Emmanuel},
  journal = {Journal of Machine Learning Research},
  volume  = {24},
  number  = {210},
  pages   = {1--55},
  year    = {2023}
}

@inproceedings{malkin2022trajectory,
  title     = {Trajectory Balance: Improved Credit Assignment in {GFlowNets}},
  author    = {Malkin, Nikolay and Jain, Moksh and Bengio, Emmanuel and Sun, Chen and Bengio, Yoshua},
  booktitle = {Advances in Neural Information Processing Systems},
  volume    = {35},
  pages     = {5955--5967},
  year      = {2022}
}

@misc{yu2025for,
  title         = {Flow of Reasoning: Training {LLM}s for Divergent Reasoning with Minimal Examples},
  author        = {Yu, Fangxu and Jiang, Lai and Kang, Haoqiang and Hao, Shibo and Qin, Lianhui},
  year          = {2025},
  eprint        = {2406.05673},
  archivePrefix = {arXiv},
  primaryClass  = {cs.AI},
  note          = {ICML 2025}
}

@inproceedings{bucila2006model,
  title     = {Model Compression},
  author    = {Buciluǎ, Cristian and Caruana, Rich and Niculescu-Mizil, Alexandru},
  booktitle = {Proceedings of the 12th ACM SIGKDD International Conference on Knowledge Discovery and Data Mining},
  pages     = {535--541},
  year      = {2006},
  doi       = {10.1145/1150402.1150464}
}

@misc{hinton2015distilling,
  title         = {Distilling the Knowledge in a Neural Network},
  author        = {Hinton, Geoffrey and Vinyals, Oriol and Dean, Jeff},
  year          = {2015},
  eprint        = {1503.02531},
  archivePrefix = {arXiv},
  primaryClass  = {stat.ML}
}

@inproceedings{agarwal2024onpolicy,
  title     = {On-Policy Distillation of Language Models: Learning from Self-Generated Mistakes},
  author    = {Agarwal, Rishabh and Vieillard, Nino and Zhou, Yongchao and Stanczyk, Piotr and Ramos, Sabela and Geist, Matthieu and Bachem, Olivier},
  booktitle = {International Conference on Learning Representations},
  year      = {2024}
}

@misc{penaloza2026privileged,
  title         = {Privileged Information Distillation for Language Models},
  author        = {Penaloza, Emiliano and Vattikonda, Dheeraj and Gontier, Nicolas and Lacoste, Alexandre and Charlin, Laurent and Caccia, Massimo},
  year          = {2026},
  eprint        = {2602.04942},
  archivePrefix = {arXiv},
  primaryClass  = {cs.LG}
}

@misc{ye2026opcd,
  title         = {On-Policy Context Distillation for Language Models},
  author        = {Ye, Tianzhu and Dong, Li and Wu, Xun and Huang, Shaohan and Wei, Furu},
  year          = {2026},
  eprint        = {2602.12275},
  archivePrefix = {arXiv},
  primaryClass  = {cs.CL}
}

@misc{xu2026betaopsd,
  title         = {{$\beta$-{OPSD}: Deriving with Policy Optimization, Training with Self-Distillation}},
  author        = {Xu, Haoran and Shi, Chufan and Yang, Chenyang and Li, Ming and Liu, Kaiyan and Zhang, Weinan and Wang, Jun},
  year          = {2026},
  eprint        = {2607.23787},
  archivePrefix = {arXiv},
  primaryClass  = {cs.AI}
}

@inproceedings{guan2026reasoningoverconfidence,
  title     = {Beware of Reasoning Overconfidence: Pitfalls in the Reasoning Process for Multi-solution Tasks},
  author    = {Guan, Jiannan and Chen, Qiguang and Qin, Libo and Peng, Dengyun and Liu, Jinhao and Huo, Liangyu and Xie, Jian and Che, Wanxiang},
  booktitle = {Proceedings of the AAAI Conference on Artificial Intelligence},
  volume    = {40},
  number    = {36},
  pages     = {30843--30851},
  year      = {2026},
  doi       = {10.1609/aaai.v40i36.40342}
}

@inproceedings{kujawa2025neuralalgorithmic,
  title     = {Neural Algorithmic Reasoning with Multiple Correct Solutions},
  author    = {Kujawa, Zeno and Poole, John and Georgiev, Dobrik and Numeroso, Danilo and Fleischmann, Henry and Li{\`o}, Pietro},
  booktitle = {Workshop on Machine Learning on Graphs in the Era of Generative Artificial Intelligence at KDD},
  year      = {2025},
  eprint    = {2409.06953},
  archivePrefix = {arXiv},
  primaryClass  = {cs.LG}
}

@inproceedings{ren2025learningdynamics,
  title     = {Learning Dynamics of {LLM} Finetuning},
  author    = {Ren, Yi and Sutherland, Danica J.},
  booktitle = {International Conference on Learning Representations},
  year      = {2025},
  eprint    = {2407.10490},
  archivePrefix = {arXiv},
  primaryClass  = {cs.LG}
}

@misc{zhou2025evolving,
  title         = {Evolving Language Models without Labels: Majority Drives Selection, Novelty Promotes Variation},
  author        = {Zhou, Yujun and Liang, Zhenwen and Liu, Haolin and Yu, Wenhao and Panaganti, Kishan and Song, Linfeng and Yu, Dian and Zhang, Xiangliang and Mi, Haitao and Yu, Dong},
  year          = {2025},
  eprint        = {2509.15194},
  archivePrefix = {arXiv},
  primaryClass  = {cs.LG}
}

@misc{yu2025rlev,
  title         = {Every Question Has Its Own Value: Reinforcement Learning with Explicit Human Values},
  author        = {Yu, Dian and Zhao, Yulai and Panaganti, Kishan and Song, Linfeng and Mi, Haitao and Yu, Dong},
  year          = {2025},
  eprint        = {2510.20187},
  archivePrefix = {arXiv},
  primaryClass  = {cs.LG}
}

@misc{liang2025g2rl,
  title         = {Can {LLM}s Guide Their Own Exploration? Gradient-Guided Reinforcement Learning for {LLM} Reasoning},
  author        = {Liang, Zhenwen and Lu, Sidi and Yu, Wenhao and Panaganti, Kishan and Zhou, Yujun and Mi, Haitao and Yu, Dong},
  year          = {2025},
  eprint        = {2512.15687},
  archivePrefix = {arXiv},
  primaryClass  = {cs.LG}
}

@misc{xia2026dirl,
  title         = {Reasoning or Memorization? Direction-Aware Diversity Exploration in {LLM} Reinforcement Learning},
  author        = {Xia, Jiangnan and Shi, Yucheng and Yang, Yu and Panaganti, Kishan and Liang, Zhenwen and Liu, Ninghao},
  year          = {2026},
  eprint        = {2606.10346},
  archivePrefix = {arXiv},
  primaryClass  = {cs.AI}
}

@misc{panaganti2026gdro,
  title         = {Group Distributionally Robust Optimization-Driven Reinforcement Learning for {LLM} Reasoning},
  author        = {Panaganti, Kishan and Liang, Zhenwen and Yu, Wenhao and Mi, Haitao and Yu, Dong},
  year          = {2026},
  eprint        = {2601.19280},
  archivePrefix = {arXiv},
  primaryClass  = {cs.LG}
}

@inproceedings{xu2025robustllm,
  title         = {Robust {LLM} Alignment via Distributionally Robust Direct Preference Optimization},
  author        = {Xu, Zaiyan and Vemuri, Sushil and Panaganti, Kishan and Kalathil, Dileep and Jain, Rahul and Ramachandran, Deepak},
  booktitle     = {Advances in Neural Information Processing Systems},
  year          = {2025},
  eprint        = {2502.01930},
  archivePrefix = {arXiv},
  primaryClass  = {cs.LG}
}

@inproceedings{panaganti2022robust,
  title         = {Robust Reinforcement Learning using Offline Data},
  author        = {Panaganti, Kishan and Xu, Zaiyan and Kalathil, Dileep and Ghavamzadeh, Mohammad},
  booktitle     = {Advances in Neural Information Processing Systems},
  year          = {2022},
  eprint        = {2208.05129},
  archivePrefix = {arXiv},
  primaryClass  = {cs.LG}
}

@misc{yu2025rfew,
  title         = {Guided Self-Evolving LLMs with Minimal Human Supervision},
  author        = {Wenhao Yu and Zhenwen Liang and Chengsong Huang and Kishan Panaganti and Tianqing Fang and Haitao Mi and Dong Yu},
  year          = {2025},
  eprint        = {2512.02472},
  archivePrefix = {arXiv},
  primaryClass  = {cs.CL},
  url           = {https://arxiv.org/abs/2512.02472}
}
